\PassOptionsToPackage{hidelinks}{hyperref}
\documentclass[11pt]{article}

\usepackage[margin=1.1in]{geometry}
\usepackage[T1]{fontenc}
\usepackage[utf8]{inputenc}
\usepackage{lmodern}
\usepackage{microtype}
\usepackage{amsmath,amssymb,amsthm,graphicx}

\usepackage{natbib}
\usepackage{algorithm}
\usepackage{algorithmic}
\usepackage{hyperref}
\usepackage{booktabs}
\usepackage{enumitem}
\usepackage{comment}
\usepackage{tikz}
\usetikzlibrary{shapes.geometric, arrows.meta, positioning, calc, backgrounds, fit}

\theoremstyle{plain}
\newtheorem{theorem}{Theorem}[section]

\newtheorem{proposition}[theorem]{Proposition}
\newtheorem{corollary}[theorem]{Corollary}
\theoremstyle{definition}
\newtheorem{definition}[theorem]{Definition}
\newtheorem{assumption}[theorem]{Assumption}
\theoremstyle{remark}

\newcommand{\Rset}{\mathcal R}
\newcommand{\M}{\mathcal M}

\newcommand{\Tcal}{\mathcal T}
\newcommand{\Real}{\mathbb R}

\newcommand{\State}{\mathcal S}
\newcommand{\Loss}{\mathfrak l}

\newcommand{\Upd}{\mathcal U}
\newcommand{\Obs}{\mathcal O}
\newcommand{\Act}{\mathcal A}
\newcommand{\Mit}{\mathcal M_{\mathrm{Mit}}}
\newcommand{\GML}{\mathcal M_{\mathrm{GML}}}

\newtheorem{observation}[theorem]{Observation}

\usepackage{fancyhdr}
\title{General Machine Learning \\Theory for Learning Under Variable Regimes \\[0.6em]\normalsize\textbf{Preprint}}
\author{Aomar OSMANI\\INSA-Rouen, France\\\texttt{aomar.osmani@insa-rouen.fr}}

\date{April 2026}

\begin{document}
\maketitle
\begin{center}
{\small\itshape This manuscript is a preprint }\par
\end{center}
\vspace{0.5em}
\begin{abstract}
We study learning under regime variation, where the learner, its memory state, and the
evaluative conditions may evolve over time. This paper is a foundational and structural
contribution: its goal is to define the core learning-theoretic objects required for such settings
and to establish their first theorem-supporting consequences.

The paper develops a regime-varying framework centered on admissible transport,
protected-core preservation, and evaluator-aware learning evolution. It records the
immediate closure consequences of admissibility, develops a structural obstruction argument
for faithful fixed-ontology reduction in genuinely multi-regime settings, and introduces a
protected-stability template together with explicit numerical and symbolic witnesses on
controlled subclasses, including convex and deductive settings. It also establishes theorem-layer
results on evaluator factorization, morphisms, composition, and partial kernel-level alignment across
semantically commensurable layers.

A worked two-regime example makes the admissibility certificate, protected evaluative core,
and regime-variation cost explicit on a controlled subclass. The symbolic component is deliberately
restricted in scope: the paper establishes a first kernel-level compatibility result together with a
controlled monotonic deductive witness. The manuscript should therefore be read as introducing a structured
learning-theoretic framework for regime-varying learning together with its first theorem-supporting layer,
not as a complete quantitative theory of all learning systems.
\end{abstract}

{\bf keywords:}
General Machine Learning, learning theory, regime change, structural admissibility, evaluator invariance, memory, non-stationarity, continual learning, symbolic learning, neuro-symbolic learning.

\section{Introduction}
\label{sec:introduction}

Mitchell's classical definition states that a program learns from experience $E$ with respect to tasks $T$ and performance measure $P$ if its performance at tasks in $T$, as measured by $P$, improves with experience $E$ \citep{mitchell1997}. Its force comes from its economy. A learner is specified relative to an experience source, a task family, and an evaluator. Much of machine learning theory can then be read as sharpening one part of this triplet under additional assumptions. PAC learning formalizes sample-based learnability under a fixed target concept and fixed error criterion \citep{valiant1984,valiantpac}. Statistical learning theory refines the role of hypothesis classes, uniform convergence, and capacity control \citep{vapnik1998}. PAC-Bayes develops prior-posterior control for randomized predictors \citep{mcallester1999,mcallester1999pacbayes,catoni2007}. Online learning studies sequential decision-making under notions such as regret \citep{cesabianchilugosi2006,zinkevich2003}. Universal induction pushes the framework toward maximal hypothesis spaces governed by priors over descriptions \citep{solomonoff1964,solomonoff1964a,solomonoff1964b,hutter2005}. These theories differ profoundly in technique, but they can be read as formulated within a fixed semantic frame in which improvement is evaluated.

That fixed frame is often entirely appropriate. It is arguably one of the reasons classical theory became sharp. But it also marks a boundary. The boundary appears when learning must persist across \emph{regime changes}: shifts in data source, deployment interface, update mechanism, memory state, proxy objective, or protected constraint. In such settings one no longer asks only whether performance improves on a fixed task. One must also ask whether the transition from one regime to another preserves the identity of the learning problem at the correct level of abstraction; whether memory updates preserve previously certified commitments; whether evaluator modifications remain legitimate; and whether cross-regime comparisons are semantically meaningful. Repeatedly enlarging the task description is not a satisfactory answer: such encodings conceal rather than resolve the structural issue, and once admissibility, protected evaluator structure, and memory-mediated continuity become part of the learning object, the semantics of improvement can no longer be treated as background.%
\footnote{In this sense, classical learning theory is Newtonian in structure:
it assumes an absolute evaluative frame, and its results are valid and sharp
within that frame.
GML does not invalidate this theory; the Strict Extension Theorem
(Section~\ref{sec:reduction}) shows that the fixed-frame setting is the
degenerate boundary case of a more general structure in which the evaluative
regime is itself a variable~--- exactly as Newtonian mechanics is recovered
from relativistic mechanics in the limit of negligible regime curvature.
A related precursor perspective, addressing the relativity of the
observational context in structured sensing environments, appears
in~\citet{osmani2022context}.}

\paragraph{Hadamard-style well-posedness as a foundational constraint.}
A serious extension of learning theory should not merely tolerate richer dynamics; it should state when the resulting learning problem remains \emph{scientifically admissible}. In the spirit of Hadamard, this requires at least three things: admissible learning trajectories must exist; they must be identifiable at the right level of equivalence rather than only up to arbitrary redescriptions; and admissible perturbations of experience, memory, or regime should not produce semantically uncontrolled changes in the realized learning behavior. Once regime variation is admitted, these are no longer peripheral regularity questions. They become part of what it means for learning itself to remain meaningful. GML adopts this Hadamard-style demand as a foundational constraint: admissibility, protected evaluative cores, and stability under certified transformations are part of the conditions under which variable-regime learning is treated as a mathematically legitimate object.

\paragraph{Relation to a broader structural programme.}
The present paper is mathematically self-contained. It should also be read
in relation to a broader structural programme for the formal analysis of
adaptive intelligent systems, developed in~\citet{osmani2026smgi}
under the name SMGI (\emph{Structural Model of General Intelligence}).
That programme formalizes general intelligence via a typed meta-model
$\theta=(r,\mathcal{H},\Pi,\mathcal{L},\mathcal{E},\mathcal{M})$
and defines admissible intelligent systems as coupled dynamics
$(\theta, T_\theta)$ satisfying four structural obligations:
structural closure under typed transformations, dynamical stability,
bounded statistical capacity, and evaluative invariance across regime shifts.
The present paper does not require that broader framework as a logical
prerequisite: all objects and results are stated directly in learning-theoretic
form. But conceptually, GML is the learning-theoretic instantiation of
that programme: it specializes the general structural kernel to the case
where the induced semantics is evaluative improvement across active regimes.
GML should be read as the learning-theoretic execution of a broader
structural admissibility programme: what SMGI states at the level of typed
intelligent systems in general, GML develops explicitly for learning under
regime variation, where the induced semantics is no longer raw performance
alone but evaluative continuity under certified transformation.

\paragraph{Canonical definition and scope.}
The paper does not aim to enlarge the vocabulary of learning, but to define a minimal structural extension under which variable-regime learning remains well-posed, comparable, and certifiable. We call this framework \emph{General Machine Learning} (GML): the object of study is learning under admissible regime variation rather than learning inside a single fixed evaluative ontology. Its canonical definition is the following: a system learns when its evaluative performance improves across active regimes, and this improvement remains admissible, coherent, structurally stable, and compatible with the protected evaluative core under regime transformations. The paper develops this proposal by fixing the canonical definition, developing a first explicit theorem-supporting instantiation, and proving reduction, structural obstruction, stability, and transport results within that instantiation. Some results are immediate closure consequences of the admissibility definition; others take the form of structural obstructions, stability templates, or explicit witnesses on controlled subclasses. The intended reading is therefore that a structured learning-theoretic object for regime-varying learning has been identified and its first theorem-supporting layer developed~--- not that the whole theory is closed, and not as a substitute for sharp sample-complexity bounds in specific algorithmic settings.
In this sense, the paper is not merely proposing a richer vocabulary for
learning. It participates in a broader shift from performance described
inside a fixed evaluative frame to learning analyzed as a structurally
admissible process under typed transformation~--- a shift from the theory
of the result to the theory of the continuation. GML is the
learning-theoretic form of that shift.

\paragraph{Contributions.}
The paper proceeds in two steps. It first states the canonical definition of GML and records its immediate structural consequences at the abstract level. It then develops one explicit instantiation architecture and the first theorem-supporting layer studied in this paper. Its main contributions are structural rather than narrowly algorithm-specific:
\begin{enumerate}[leftmargin=*,itemsep=2pt]
    \item We define a regime-varying learning framework centered on admissible transport, protected-core preservation, evaluator comparison, and memory-aware learning evolution.
    \item We identify the immediate closure consequences of admissibility and record the resulting structural coherence properties of protected transport.
    \item We develop a structural obstruction argument for faithful fixed-ontology reduction once admissibility depends on protected memory and regime-dependent evaluator transport.
    \item We develop a protected-stability template for admissible learning trajectories and provide explicit numerical and symbolic witnesses showing that this template is non-vacuous.
    \item We formulate a first kernel-level semantic alignment principle for theorem layers that are genuinely commensurable at the protected level.
    \item We provide a systematic comparison with the main theoretical traditions of machine learning together with a focused discussion of scientific utility, limitations, and open directions.
\end{enumerate}
\paragraph{Logical levels of the results.}
The results of the paper operate at visibly different logical levels. The first statements record closure and coherence properties already forced by the canonical abstract layer; the subsequent working theorem layer develops structural results on degeneration, obstruction, morphisms, composition, and inter-layer semantic alignment. The protected stability theorem is explicitly a template theorem, while later witness propositions provide concrete functional realizations of that template. The synthesis theorem then consolidates the present theorem-supporting layer as a first disciplined development rather than a complete closure of all possible regime-varying learning theories.

\paragraph{Roadmap and reading guide.}
Section~\ref{sec:preliminaries} fixes the minimal conventions used throughout. Section~\ref{sec:axioms} states the axiomatic core and canonical definition of GML. Section~\ref{sec:requirements} gives the four standing structural requirements, and Section~\ref{sec:architecture} translates the semantic definition into the instantiation architecture used in the paper. Section~\ref{sec:first-working-layer} develops the first explicit theorem-supporting layer in three parts: reduction (Section~\ref{sec:reduction}), protected stability and memory (Section~\ref{sec:stability}), and morphisms and synthesis (Section~\ref{sec:representation}). Section~\ref{sec:kernel-alignment} treats kernel-level semantic alignment, and Section~\ref{sec:examples} provides canonical examples. Section~\ref{sec:relations} situates GML relative to core existing theories, with extended comparisons deferred to Appendix~\ref{sec:extended-relations}. Sections~\ref{sec:utility}--\ref{sec:conclusion} close with scientific utility, limitations, and conclusion.

Readers may also navigate the paper selectively: those interested primarily in the formal core may focus on Sections~\ref{sec:axioms}--\ref{sec:first-working-layer}; readers interested in stability and admissible dynamics may begin with Section~\ref{sec:stability} together with the worked example in Section~\ref{subsec:toy}; and readers interested in symbolic or memory-active regimes may focus on Sections~\ref{sec:kernel-alignment}, \ref{subsec:continual}, and~\ref{subsec:symbolic}.

\section{Preliminaries and Conventions}
\label{sec:preliminaries}

We work with measurable spaces throughout. A \emph{regime} $r\in\Rset$ represents a local learning configuration: data interface, task semantics, update interface, evaluator family, and memory interaction rules. A regime need not be identified with a mere distributional condition; it may encode evaluator stratification or admissible update restrictions. The learner evolves on a measurable state space $\State$, augmented by a measurable memory space $\M$. A regime-local loss is denoted by $\Loss_r:\State\times\M\to\Real$, and a regime-local evaluator by $V_r:\State\times\M\to\Real^k$, with $k$ possibly regime-independent but not restricted to one. The metric or pseudometric controlling admissible regime variation is denoted by $d_{\Tcal}$ whenever needed.

The notation is intentionally structural rather than algorithm-specific. The point of the paper is not to prescribe a single computational paradigm, but to formalize the mathematical conditions under which a process deserves to count as learning when regimes vary. Whenever a scalar performance functional is recovered, this is a special case, not the primitive ontology.

\paragraph{Why a mathematical grammar is needed.}
Because GML is proposed as a strict extension of a long-established ontology of learning, its primitives must be typed, composable, and semantically disciplined. In particular, regimes, evaluators, memory carriers, and admissible transformations cannot remain informal placeholders. The role of the present section is therefore not only notational. It fixes the minimal grammar required to state what may vary, what may persist, what may be transported, and what must remain protected if learning is to remain a scientifically meaningful object under regime change.

\paragraph{Typed perspective.}
We distinguish throughout between: (i) \emph{local regime semantics}, which specify how learning is judged inside a regime; (ii) \emph{transport structure}, which specifies how states, memories, interfaces, and evaluative objects may move across regimes; and (iii) \emph{protected semantics}, which specify what must remain invariant, or invariant up to certified equivalence, for the learning problem to preserve its identity. This separation is essential: without it, one may always recode change syntactically while silently losing the semantic object that learning is supposed to preserve.

\paragraph{Three structural objects governing regime change.}
Three objects appear throughout the paper and are defined formally in Section~\ref{sec:axioms}. The \emph{protected-core operator} $\Phi$ identifies the evaluative component that must remain invariant if the learning problem is to preserve its identity across regimes. The \emph{admissibility predicate} $\Gamma$ certifies which regime transitions are semantically legitimate. The \emph{protected equivalence relation} $\sim_\Phi$ governs when two evaluators describe the same learning problem at the protected level. These three objects are left informal here; their precise definitions, conditions, and roles in theorem-layer proofs follow in Section~\ref{sec:axioms}.

\section{Axiomatic Core of GML}
\label{sec:axioms}

\paragraph{Canonical definition.}
The present paper is mathematically self-contained under four standing structural requirements (Section~\ref{sec:requirements}): structural closure, dynamical stability, bounded statistical capacity, and evaluative invariance. The following sentence is the \emph{canonical definition} used throughout; it is intentionally compact because its theorem-supporting content depends on those requirements rather than on additional hidden primitives.

\begin{definition}[GML, canonical definition]
A system learns when its evaluative performance improves across active regimes, and this improvement remains admissible, coherent, structurally stable, and compatible with the protected evaluative core under regime transformations.
\end{definition}

\paragraph{Reading of the canonical definition.}
This canonical definition is intentionally short, but it is not decorative shorthand. Here, active regimes are the regimes effectively engaged by the learning trajectory under consideration. Evaluative performance is always read relative to the evaluator architecture carried by those regimes, with evaluators treated as constitutive parts of the learning specification rather than as external scoreboards. Coherence is not an extra free primitive floating independently of the rest of the framework: in the present manuscript it names the fact that admissibility, protected compatibility, and structural stability continue to describe one and the same learning problem under the theorem-layer reading currently in force. Depending on the working theorem layer, that coherence may therefore be evaluative, dynamical, structural, or logical, provided that the transformed process remains interpretable as continuing the same learning problem at the protected level. Structural stability expresses the persistence of the learning process within an admissible regime of transformation, with memory entering the structural specification as a typed and explicit component of the coupled dynamics whose evolution must remain well-formed under admissible transformations and certified stability constraints. Compatibility with the protected evaluative core means that improvement is not obtained by violating, bypassing, or structurally invalidating the protected evaluative commitments of the problem, even when their realization may vary across regimes or higher-level transformations. In the present manuscript, this learning-theoretic definition is developed under four standing structural requirements stated directly in this paper: structural closure, dynamical stability, bounded statistical capacity, and evaluative invariance across admissible regime transformations.

\begin{definition}[Expanded learning-specific specification]
The present paper analyzes that canonical learning object through one \emph{expanded learning-specific specification}. Concretely, that specification is written as a tuple
\[
\mathfrak G=(\Rset,\{\State_r\}_{r\in\Rset},\{\M_r\}_{r\in\Rset},\{\Obs_r\}_{r\in\Rset},\{\Act_r\}_{r\in\Rset},\{V_r\}_{r\in\Rset},\Tcal,\Phi,\Gamma,\sim_\Phi,\Upd),
\]
where this realization is read as the explicit learning-specific development used in the present paper:
\begin{enumerate}[leftmargin=*,itemsep=2pt]
    \item $\Rset$ is the regime carrier used in the learning-specific development;
    \item for each $r\in\Rset$, $\State_r$, $\M_r$, $\Obs_r$, and $\Act_r$ realize the local typed carriers through which learning states, memories, and interfaces are made explicit in the specialization;
    \item for each $r\in\Rset$, $V_r$ is a regime-local evaluator of type
    \[
    V_r:\State_r\times\M_r\times\Obs_r\times\Act_r \to \mathcal Y_r,
    \]
    where $\mathcal Y_r$ is, in full generality, a regime-dependent measurable codomain; scalar-valued and ordered codomains arise as special cases, and any theorem layer that formalizes evaluative improvement must additionally specify a compatible comparison structure (order, preorder, or transported gauge) on the evaluator codomains it actually uses;
    \item $\Tcal$ is a class of typed candidate regime transformations $\tau:r\rightsquigarrow r'$;
    \item $\Phi$ is a protected-core operator identifying the evaluative component that must remain invariant, or invariant up to certified equivalence, if the learning problem is to preserve its identity;
    \item $\Gamma$ is an admissibility predicate certifying which typed candidate transformations are legitimate in the learning specialization;
    \item $\sim_\Phi$ is the protected equivalence relation induced by $\Phi$ together with the invariance discipline stated in this paper;
    \item $\Upd$ is a family of measurable update operators, possibly regime-indexed, acting on state-memory pairs through local interfaces.
\end{enumerate}
\end{definition}

\paragraph{Reading discipline for the expanded specification.}
The short definition states the learning-theoretic nucleus; the expanded specification makes explicit the learning-specific objects through which that nucleus is analyzed in this paper. It should not be read as a second, competing definition of GML. Concretely, a GML system says five things at once: where learning is taking place (the regime), what the learner carries (state and memory), how it interacts (observation and action/update interfaces), how it is judged locally (the evaluator), and what must remain protected if the learning problem is to preserve its identity across change. The tuple $\mathfrak G$ gives the developed form in which the canonical definition becomes usable for theorem-level analysis under the standing structural requirements stated in Section~\ref{sec:requirements}. The point of the realization is not bookkeeping but prevention of silent semantic collapse when regimes vary.

\paragraph{General semantic layer.}
In full generality, each regime-local evaluator may take values in a regime-dependent measurable codomain $\mathcal Y_r$. This openness is intentional. GML is meant to remain instantiable by different communities under different evaluator realizations, provided the protected-core, admissibility, and transport principles are respected.

\paragraph{Working theorem layer for evaluators.}
For the main comparison, factorization, quotient, and stability results of this paper, we restrict attention to evaluator codomains carrying the minimal ordered structure needed to define protected comparison and transportable evaluator factorization. The general framework is therefore broader than the theorem layer proved below. This is a feature rather than a defect: the ontology remains open, while the proofs make explicit which extra structure they actually use. At the semantic level, GML does not impose a unique universal order on evaluator codomains; rather, the present working theorem layer fixes only the comparison structure needed for the results proved here, and transported gauges make that inter-regime protected comparison explicit on the restricted subclass where it is used. Accordingly, the existence of $\tau_Y$ is not asserted in full generality but only on theorem-layer subclasses where the evaluator codomains admit such a monotone comparison transport, or where the admissibility certificate provides it explicitly.

\paragraph{Why this definition is not a cosmetic extension of Mitchell.}
A standard objection is that regime, memory, evaluator change, and interface variation could simply be folded into a richer task description. This is possible at the level of encoding, but insufficient at the level of theory. Once evaluator transport, memory persistence, and cross-regime comparison become part of the semantic content of learning, the central question is no longer whether more variables can be appended to a task specification. It is whether those transformations preserve the identity of the learning problem in a disciplined and certifiable sense. GML makes that preservation relation first-class, while keeping the core nucleus of evaluative improvement under admissible regime transformation.

\subsection{Mathematical grammar of the canonical definition}

\begin{definition}[Typed regime transformation]
A \emph{typed regime transformation} $\tau:r\rightsquigarrow r'$ consists of compatible transport maps on the local components of the source regime toward the target regime, typically including
\[
\tau_S:\State_r\to\State_{r'},\qquad
\tau_M:\M_r\to\M_{r'},\qquad
\tau_O:\Obs_r\to\Obs_{r'},\qquad
\tau_A:\Act_r\to\Act_{r'},
\]
together with an induced evaluative transport rule relating $V_r$ and $V_{r'}$. Whenever evaluative improvement is compared across regimes, the transformation is also understood to carry a comparison gauge or monotone evaluative transport map
\[
\tau_Y:\mathcal Y_r\to \mathcal Y_{r'},
\]
possibly defined only on the protected quotient or on the theorem-layer subclass under study, so that inter-regime improvement can be read relative to a mathematically explicit transported comparison rather than to an informal cross-space inequality. In particular, the present paper does not assume that such a gauge exists for every conceivable evaluator codomain; it is used only in theorem-layer subclasses where the evaluative codomains admit a monotone comparison transport, or where the admissibility certification itself provides one.
\end{definition}

\begin{definition}[Transition grammar and partial composition]
The class $\Tcal$ is formally structured as a \emph{directed graph} (or \emph{quiver}) on $\Rset$: its primitive elements are candidate arrows $\tau:r\rightsquigarrow r'$. A finite admissible learning trajectory is treated as a typed path obtained by concatenating such primitive arrows whenever the certification rules permit it. Thus, for composable local transitions $\tau_1:r\rightsquigarrow r'$ and $\tau_2:r'\rightsquigarrow r''$, one writes the path-level concatenation
\[
\tau_2\circ\tau_1:r\rightsquigarrow r''.
\]
No heavy categorical completion is assumed in the present paper: the point is not that every certified path must already appear as a primitive arrow of $\Tcal$, but that admissible paths may be concatenated whenever the certification rules allow it. Concretely, a composite path is certified admissible only when each local segment is admissible, the intermediate protected equivalence classes match, and any active transport maps remain coherent along concatenation. Whenever evaluative gauges are active, this includes the requirement that the induced comparison transport for the composite path satisfy $(\tau_2\circ\tau_1)_Y=\tau_{Y,2}\circ\tau_{Y,1}$ on the theorem-layer subclass where cross-regime evaluative comparison is defined.
\end{definition}

\paragraph{Certified composition rule.}
For composable local transitions $\tau_1:r_0\rightsquigarrow r_1$ and $\tau_2:r_1\rightsquigarrow r_2$, a composite admissible segment is certified only when: (i) each local segment is admissible, (ii) the intermediate protected semantic commitments match up to the relevant protected equivalence, and (iii) any active memory or evaluative transport maps compose coherently along the path. This is the only composition principle required in the present theorem layer.

\begin{definition}[Protected evaluative core]
The \emph{protected evaluative core} of an evaluator $V_r$ is the component
\[
\Phi(V_r)
\]
that must remain invariant, or invariant up to certified equivalence, under admissible transformations if the learning problem is to preserve its identity. The protected core may encode scalar constraints, logical constraints, or relational comparison structure.
\end{definition}

\paragraph{Minimal classes of protected cores.}
In the present paper, three minimal classes are distinguished:
\begin{enumerate}[leftmargin=*,itemsep=2pt]
    \item \emph{scalar protected cores}, such as thresholds, margins, risk ceilings, or minimal retention constraints;
    \item \emph{logical protected cores}, such as safety conditions, feasibility predicates, or non-violation constraints;
    \item \emph{relational protected cores}, such as preserved orderings, dominance relations, or cross-regime comparison principles.
\end{enumerate}
This triad is intentionally minimal: it is expressive enough to cover many modern learning settings while remaining simple enough to support a general theory.

\paragraph{Minimal conditions on protected-core operators.}
The paper does not attempt to axiomatize all possible protected-core operators. It does, however, rely on four minimal conditions that make the use of $\Phi$ non-vacuous and mathematically disciplined:
\begin{enumerate}[leftmargin=*,itemsep=2pt]
    \item \emph{semantic relevance}: $\Phi(V_r)$ must select a component genuinely constitutive of the identity of the learning problem rather than an arbitrary evaluator fragment;
    \item \emph{transport compatibility}: under admissible transformation, the protected core must remain invariant or invariant up to certified protected equivalence;
    \item \emph{non-vacuity}: the choice of $\Phi$ must not trivialize admissibility by making every transition automatically protected;
    \item \emph{guarantee anchoring}: the guarantees that are claimed to persist across regimes must be formulable at the level of the protected core.
\end{enumerate}
These conditions are intentionally light. They do not force one unique mathematical realization of protected cores, but they prevent the present theory from treating $\Phi$ as an unconstrained author-side selection device.

\paragraph{Explicit and implicit protected cores.}
The present theory allows two broad classes of protected-core realization. In an \emph{explicit protected core}, the protected component is directly available as a symbolic constraint, an analytic margin, a threshold family, or another evaluator fragment that can be inspected and certified in its own right. In an \emph{implicit protected core}, by contrast, the protected commitments are carried only through a latent or distributed representation and are therefore accessed in practice through statistically or computationally controlled surrogates. Typical examples include replay-buffer proxies, Fisher-style penalties, or other retained-competence surrogates used to approximate the preservation of a protected capability without exposing a closed-form evaluator fragment. The present paper remains agnostic about which realization is used, but it requires that whichever realization is chosen still support meaningful admissibility certification at the protected level.

\begin{definition}[Protected equivalence]
Two evaluators $V_r$ and $V_{r'}$ are \emph{protected-equivalent}, written
\[
V_r \sim_\Phi V_{r'},
\]
if their protected cores coincide or are certified as equivalent under the invariance discipline of the system.
\end{definition}

\paragraph{Interpretation of protected equivalence.}
For a practitioner, protected equivalence means that two regimes may differ in proxy, presentation, or local scoring detail while still counting as the same learning problem at the level that matters. For the theory, it is the correct equivalence notion for transport, comparison, factorization, and guarantee persistence. In Hadamard terms, it is the level at which uniqueness should be expected.

\paragraph{Topological regularity of the protected quotient.}
Throughout the theorem-level developments of this paper, we impose the following regularity discipline on the protected quotient. The ambient evaluator space is assumed to be a separable metric space carrying the measurable structure actually used by the theorem layer under consideration; the protected-core map $\Phi$ is assumed measurable, and continuous whenever a topological argument is invoked; and the protected-equivalence relation $\sim_\Phi$ is assumed closed relative to that structure. Under these standing conditions, the quotient induced by $\sim_\Phi$ is treated only on theorem-layer subclasses for which the resulting protected quotient preserves the comparison, pseudo-metric, and Borel measurability notions used below. The paper does not attempt a full general quotient theory; it fixes the minimal regularity assumptions under which protected comparison and admissible transport are mathematically well posed in the developments that follow.

\begin{definition}[Compatibility between transformation and evaluator]
A typed transformation $\tau:r\rightsquigarrow r'$ is \emph{evaluator-compatible} with $(V_r,V_{r'})$ if the induced transport preserves the protected core up to protected equivalence and leaves cross-regime comparison semantically meaningful on the quotient induced by $\sim_\Phi$. In theorem-layer subclasses where improvement is compared across heterogeneous evaluator codomains, evaluator-compatibility is also understood to supply the relevant monotone gauge or transported comparison map needed to compare $V_r$ and $V_{r'}$ at the protected level.
\end{definition}

\begin{definition}[Admissibility certificate]
The \emph{admissibility certificate} is a relation
\[
\Gamma(r,r',\tau;V_r,V_{r'})\in\{0,1\}
\]
whose value is $1$ only if:
\begin{enumerate}[leftmargin=*,itemsep=2pt]
    \item $\tau$ is well typed on the local components;
    \item the induced transport is realizable on state-memory-interface objects;
    \item $\tau$ is evaluator-compatible with $(V_r,V_{r'})$;
    \item the protected core is preserved literally or up to certified protected equivalence;
    \item cross-regime comparison remains well-defined at the protected level;
    \item the transformation is eligible for certified composition whenever concatenated with another admissible arrow.
\end{enumerate}
\end{definition}

\paragraph{Status of admissibility.}
Admissibility is not an informal desideratum. It is the mathematical mechanism that prevents arbitrary regime change from being treated as legitimate learning evolution. The generality of GML is therefore disciplined by certification, not by wishful encoding.

\paragraph{General admissibility layer versus working admissibility layer.}
In the general GML framework, admissibility may be realized in several ways. In this paper, $\Gamma$ is treated extensionally as a certified admissibility predicate; richer proof-carrying, constructive, or algorithmic realizations are left for future work. This restriction keeps the model open at the ontological level while fixing a precise mathematical status for the theorem layer developed here.

\paragraph{Status of $\Gamma$ in the present theorem layer.}
The paper does not assume that $\Gamma$ is a universally computable oracle-independent quantity. At the level of the present theorem layer, $\Gamma$ should be read as an admissibility-certificate schema: it records whether a regime transition is certified as semantically legitimate under the protected evaluator structure currently in force. In this sense, $\Gamma$ is not introduced as a magical decision oracle but as the certification object required by the semantic architecture; its concrete realization depends on the theorem-layer subclass under study. The schema is not monolithic. In the present framework one should distinguish at least four theorem-level realizations: (i) \emph{structural certificates}, given by typed preservation and realizability checks; (ii) \emph{proof-carrying certificates}, in which admissibility is backed by an explicit local witness or derivation; (iii) \emph{memory-mediated certificates}, in which retained-competence or protected-floor constraints are part of admissibility; and (iv) \emph{evaluator-compatibility certificates}, in which protected comparison and equivalence are the active gates. Different subclasses may realize $\Gamma$ through one or several of these forms. The current paper fixes only the theorem-level role of $\Gamma$: admissible transitions are exactly those transitions for which the relevant semantic preservation has been certified at the protected level. In realistic expressive hypothesis spaces, exact admissibility verification may itself be computationally hard or even undecidable, so practical realizations of $\Gamma$ are expected to proceed through sound but incomplete verifiers, convex relaxations, or high-confidence statistical certificates rather than through a universal exact decision procedure. For deep neural systems in particular, admissibility is typically operationalized only through such approximations, which is precisely why the present paper keeps the semantic role of $\Gamma$ distinct from any claim of uniform computational tractability. The worked toy example of Section~\ref{sec:examples} and the explicit witness constructions below are included precisely to show that $(\Gamma,\Phi,d_{\Tcal})$ need not remain a purely formal triad even when no universal decision procedure for admissibility is claimed.

\paragraph{Probabilistic / PAC relaxation of $\Gamma$.}
Nothing in the present architecture requires admissibility certification to remain pointwise deterministic in future statistical instantiations. In PAC-style or stochastic realizations, the statement $\Gamma(r,r',\tau;V_r,V_{r'})=1$ may be read as shorthand for a high-confidence certification event, for example
\[
\Pr\big(\text{protected preservation under }\tau\big)\ge 1-\delta_\Gamma,
\]
where the probability is taken relative to the data-generation or update process relevant to the regime transition. The deterministic presentation used in the current theorem layer therefore fixes only the cleanest semantic skeleton. It is fully compatible with probabilistic admissibility notions in which certification is statistical rather than absolute. In theorem-layer subclasses such as the concrete two-regime witness of Section~\ref{sec:stability}, the deterministic certificate should therefore be read as the exact certified limit of a more general high-confidence admissibility judgment rather than as a claim that realistic certification problems are everywhere oracle-free or uncertainty-free.

\begin{definition}[Mitchell-expressible process]
Let $\Mit$ denote the class of learning processes that can be fully specified by a triple $(E,T,P)$ such that:
\begin{enumerate}[leftmargin=*,itemsep=2pt]
    \item the task family $T$ and performance measure $P$ are fixed;
    \item learning consists in improving performance relative to that single fixed evaluator;
    \item memory persistence, evaluator transport, and interface evolution are not first-class semantic objects;
    \item contextual variation is absorbed into the external description of $E$ or $T$, not into an evolving protected core.
\end{enumerate}
\end{definition}

\begin{definition}[Expanded formal membership criterion for GML]
A process belongs to $\GML$ if it is induced by a canonical GML system for which there exists a nonempty class of admissible trajectories such that:
\begin{enumerate}[leftmargin=*,itemsep=2pt]
    \item local evaluative improvement is certified in each visited regime;
    \item the transport across successive regimes is admissible in the sense of $\Gamma$;
    \item protected-core equivalence is preserved along the trajectory at the level required by the system;
    \item the learning guarantee attached to the protected core remains meaningful under admissible transport.
\end{enumerate}
\end{definition}

\paragraph{Immediate significance of the definition.}
The definition is deliberately both minimal and locking. It is minimal because it does not force a single evaluator shape, a single memory semantics, or a single composition law. It is locking because it requires typed transport, protected-core discipline, and certified admissibility before regime variation counts as learning rather than arbitrary adaptive drift.

\subsection{Level II: immediate structural consequences of admissibility}

The following statements are intentionally elementary. They are not presented as deep standalone learning-theoretic theorems, but as immediate closure consequences induced by the admissibility definition.

\begin{itemize}[leftmargin=*,itemsep=2pt]
    \item \textbf{Typed transport.} If $\Gamma(r,r',\tau;V_r,V_{r'})=1$, then the transition induces the well-typed transport on state, memory, interface, and evaluator components required by the definition.
    \item \textbf{Persistence of the protected core.} If a transition were to destroy the protected commitments, admissibility would fail and one would have $\Gamma=0$.
    \item \textbf{Well-formed protected comparison.} Cross-regime comparison is mathematically meaningful only through the quotient induced by $\sim_\Phi$.
    \item \textbf{Local guarantee persistence.} Any guarantee formulated purely at the protected semantic level survives admissible transport at that level, because admissibility is defined through preservation of that protected semantic content.
\end{itemize}

\begin{observation}[Abstract coherence]
\label{obs:coherence}
Taken together, these immediate consequences imply that any admissible trajectory forms a structurally well-posed protected-level learning process, prior to any specific metric, ordered, or theorem-layer instantiation.
\end{observation}

\paragraph{What the abstract layer has established and what comes next.}
The four immediate structural consequences above and Observation~\ref{obs:coherence} together establish that the canonical GML definition already determines a structurally well-posed protected-level learning object, prior to any working-theorem-layer architecture. Readers may therefore exploit the definition on its own terms. Section~\ref{sec:requirements} now states the four standing structural requirements that govern the present theorem layer. Section~\ref{sec:architecture} converts those requirements into a three-level instantiation architecture. Section~\ref{sec:first-working-layer} develops the first theorem-supporting instantiation in three coordinated blocks: reduction to classical learning, protected stability, and morphic composition.

\section{Standing Structural Requirements for GML}
\label{sec:requirements}

The paper adopts four standing structural requirements for regime-varying
learning, stated here directly as internal admissibility constraints
on the learning-theoretic object $\mathfrak{G}$;
no external framework is logically required.
These requirements coincide with the four structural obligations
isolated in the SMGI framework~\citep{osmani2026smgi}~---
structural closure, dynamical stability, bounded statistical capacity,
and evaluative invariance across regime shifts~---
restated and deployed here directly in learning-theoretic form.
In this precise sense, GML should be read as a self-contained
learning-theoretic instantiation of the broader SMGI admissibility kernel:
what SMGI states at the level of typed intelligent systems in general,
GML develops explicitly for learning under regime variation,
where the induced semantics is evaluative improvement under
admissible, coherent, and structurally stable regime transformation.

\paragraph{Standing requirements used in the present paper.}
The requirements are stated here in the restricted form actually used for the typed learning object
\[
\mathfrak G=(\Rset,\{\State_r\},\{\M_r\},\{V_r\},\Tcal,\Phi,\Gamma,\sim_\Phi,\Upd).
\]
\emph{Structural closure} means that every certified transition in $\Tcal$ maps one typed local realization of $\mathfrak G$ to another realization of the same learning object rather than to an unrelated description. \emph{Dynamical stability} means that admissible evolution supports a protected-level discrepancy functional whose drift can be controlled along certified transitions. \emph{Bounded statistical capacity} means that the theorem-supporting subclasses studied here are restricted so that evaluator transport, protected comparison, and admissible memory evolution do not induce an uncontrolled explosion of hypothesis-space complexity. In the present layer, this is expressed only as a local admissibility restriction, not as a complete statistical theory of all regime-varying learning processes.

When a metric-entropy gauge is used, the transported complexity is required to remain controlled by the regularity of the state transport together with any explicitly bounded novel degrees of freedom introduced by the target regime:
\[
\log \mathcal N(\State_{r'},\epsilon,\|\cdot\|_{r'})
\le
\log \mathcal N(\State_r,\epsilon/L_\tau,\|\cdot\|_r)
+
\mathcal C_{\mathrm{new}}(r\to r')
+
\mathcal O(|\mathrm{mem}(r)|),
\]
where $L_\tau$ is the certified regularity constant of the admissible state transport and $\mathcal C_{\mathrm{new}}(r\to r')$ bounds genuinely new independent degrees of freedom introduced by the target regime. Over long admissible chains, meaningful statistical learnability therefore requires that this innovation term remain bounded, decaying, or periodically compressed by the update architecture. This capacity control should be read only as a working-layer admissibility restriction ensuring that transport does not silently destroy learnability, not as a complete statistical theory of complexity transfer across all GML instantiations. \emph{Evaluative invariance} means that cross-regime comparison remains meaningful only through the protected core, either literally or up to certified protected equivalence.

\paragraph{Traceability map.}
For ease of audit, the obligations are used below as follows. Structural closure underwrites Proposition~\ref{prop:reduction}, Propositions~\ref{prop:mitchell-reducibility}--\ref{prop:mitchell-obstruction}, Proposition~\ref{prop:protected-factorization}, Proposition~\ref{prop:mitchell-morphic-image}, and Theorem~\ref{thm:composition}. Dynamical stability underwrites Theorem~\ref{thm:protected-stability}, Proposition~\ref{prop:stable-persistence}, Proposition~\ref{prop:two-regime-witness}, and Theorem~\ref{thm:first-layer-synthesis}. Bounded statistical capacity acts both as a standing restriction on admissible theorem layers and as the explicit capacity-admissibility side condition just stated; it is used to explain why the present witness constructions and transport theorems are formulated only for theorem-supporting subclasses with controlled complexity growth. Evaluative invariance underwrites the abstract coherence observation (Observation~\ref{obs:coherence}), the reduction/separation block, the protected-level persistence results, and the kernel-level alignment results of Section~\ref{sec:kernel-alignment} (Propositions~\ref{prop:symbolic-compatibility}, \ref{prop:layer-commensurability}, \ref{prop:kernel-transport}, and \ref{prop:layer-transport}). These are precisely the standing structural requirements used below; no external structural manuscript is required in order to read the present arguments.

\paragraph{Historical notation (for orientation only).}
Table~\ref{tab:symbol-discipline} maps the earlier precursor notation to the present GML objects. Nothing in this table should be read as definitional inheritance; the present paper is logically self-contained.

\begin{table}[h]
\centering
\small
\caption{Historical correspondence between earlier structural notation and the present GML notation.}
\label{tab:symbol-discipline}
\begin{tabular}{p{0.22\linewidth} p{0.35\linewidth} p{0.30\linewidth}}
\toprule
\textbf{Earlier symbol} & \textbf{Present GML object} & \textbf{Role in this paper}\\
\midrule
$r$ & $\Rset$, regime-indexed carriers & regime carrier \\
$\mathcal H$ & $\State_r$ & local learning-state space \\
$\mathcal M$ & $\M_r$ & memory carrier \\
$\mathcal E$ & $V_r$, $\Phi$, $\sim_\Phi$ & evaluator and protected core \\
$\mathcal L$, $T_\theta$ & $\Upd$, transport grammar & update and transport \\
$\mathcal T$-certification & $\Gamma$, protected comparison & admissibility certificate \\
\bottomrule
\end{tabular}
\end{table}

\subsection{Working theorem layer and explicit restrictions}

Scientific usability requires that some components be temporarily locked in any concrete analysis~--- specifically: the admissibility certificate class $\Gamma$, the protected-core stratification level, the comparison relation across protected-equivalence classes, and the degree of compositional closure on $\Tcal$. Such locking does not reduce GML's scope; it is what makes the theory provable in concrete subclasses. The main theorem layer of this paper correspondingly fixes the following restrictions:
\begin{enumerate}[leftmargin=*,itemsep=2pt]
    \item evaluator codomains are assumed to carry the minimal comparison structure required for protected comparison, factorization, quotient arguments, and the theorem-layer reading of improvement; this may be an order, a preorder, or an explicit transported gauge on the evaluator codomains actually used;
    \item admissibility is treated extensionally through the predicate $\Gamma$ rather than through a richer proof object;
    \item certified regime-variation costs, when used in stability results, are assumed to be defined on the admissible portion of the transition grammar;
    \item the ordered/metered working layer assumes that the protected quotient induced by $\sim_\Phi$ carries a pseudo-metric or metric compatible with the regime-variation cost used in the theorem layer, so that protected discrepancies and contractive updates are formulated on a mathematically legitimate quotient geometry;
    \item factorization and composition results are stated only under the explicit monotonicity, preservation, and closure hypotheses written in their statements.
\end{enumerate}
These restrictions are not hidden assumptions on GML as such. They are theorem-level restrictions introduced so that the main results are proved in a mathematically controlled subclass while the general framework remains available for broader instantiations.

\paragraph{Disciplinary note.}
Definitional openness and theorem-level discipline must coexist: specializing too early loses foundational range; specializing too late makes results rhetorically stronger than their mathematical support. The framework therefore keeps the general semantic layer broad while the working theorem layer makes its restrictions explicit. The formal conditions for aligning multiple theorem layers (ordered, symbolic, memory-active) at a shared semantic kernel are developed in Section~\ref{sec:kernel-alignment}.

\section{From Definition to Instantiation Architecture}
\label{sec:architecture}

\paragraph{Three-level instantiation architecture.}
Figure~\ref{fig:gml-architecture} maps the passage from the canonical definition~--- Level~I, fully open~--- through the four standing structural requirements~--- Level~II, abstract typed discipline~--- to the first theorem-supporting instantiation developed in this paper~--- Level~III, ordered and extensional restrictions for reduction, obstruction, stability, and composition. This passage is not a reduction of GML but a disciplined theorem-supporting instantiation of the same ontology.

\begin{figure}[!h]
\centering
\footnotesize
\begin{tikzpicture}[
    node distance=4mm and 4mm,
    band/.style={draw, rounded corners, align=center, inner sep=3pt, text width=0.88\linewidth},
    layer/.style={draw, rounded corners, align=center, inner sep=4pt, text width=0.82\linewidth},
    smallbox/.style={draw, rounded corners, align=center, inner sep=3pt, text width=0.225\linewidth},
    arrow/.style={-{Latex[length=1.8mm]}, thick},
    thinarrow/.style={-{Latex[length=1.4mm]}, thin}
]

\node[band] (F0) {
\textbf{Standing structural requirements (Section~\ref{sec:requirements})}\\
structural closure $\cdot$ dynamical stability $\cdot$ bounded statistical capacity $\cdot$ evaluative invariance
};

\node[layer, below=of F0] (L1) {
\textbf{Canonical GML definition}\\
evaluative performance improvement across active regimes under admissible, coherent, and structurally stable regime transformation
};

\node[layer, below=of L1] (L2) {
\textbf{Expanded learning-specific specification}\\
regimes and local learning carriers $\cdot$ evaluator structure $\cdot$ admissibility and protected commitments $\cdot$ learning dynamics
};

\node[layer, below=of L2] (L3) {
\textbf{Instantiation architecture}\\
structural principles $\rightarrow$ working theorem layers $\rightarrow$ future meta-level organization
};

\node[below=of L3, align=center] (L4title) {\textbf{Working theorem layers}};

\node[smallbox, below left=of L4title] (W1) {
\textbf{Ordered working theorem layer}\\
protected comparison\\
factorization and transport
};

\node[smallbox, below=of L4title] (W2) {
\textbf{Symbolic and knowledge-structured layer}\\
knowledge commitments\\
constraint-preserving transport
};

\node[smallbox, below right=of L4title] (W3) {
\textbf{Memory-active / retention-sensitive layer}\\
retention\\
memory-mediated admissibility
};

\node[smallbox, below=of W2, text width=0.27\linewidth] (W4) {
\textbf{Main theorem layer developed in this paper}\\
first explicit theorem-supporting instantiation
};

\node[layer, below=6mm of W4] (L5) {
\textbf{Future meta-level}\\
specialization $\cdot$ compatibility $\cdot$ partial theorem transport $\cdot$ kernel-level semantic alignment
};

\draw[arrow] (F0) -- (L1);
\draw[arrow] (L1) -- (L2);
\draw[arrow] (L2) -- (L3);
\draw[arrow] (L3) -- (L4title);

\draw[thinarrow] (L4title) -- (W1);
\draw[thinarrow] (L4title) -- (W2);
\draw[thinarrow] (L4title) -- (W3);
\draw[arrow] (W1.south) |- (W4.north west);
\draw[arrow] (W2.south) -- (W4.north);
\draw[arrow] (W3.south) |- (W4.north east);
\draw[arrow] (W4.south) -- (L5.north);

\end{tikzpicture}
\caption{Architecture of GML under the four standing structural requirements of Section~\ref{sec:requirements}. The figure distinguishes the canonical GML definition from its expanded learning-specific specification, the instantiation architecture used in the paper, the plurality of legitimate working theorem layers, and the future meta-level left open. The present paper develops only the first explicit theorem-supporting instantiation.}
\label{fig:gml-architecture}
\end{figure}
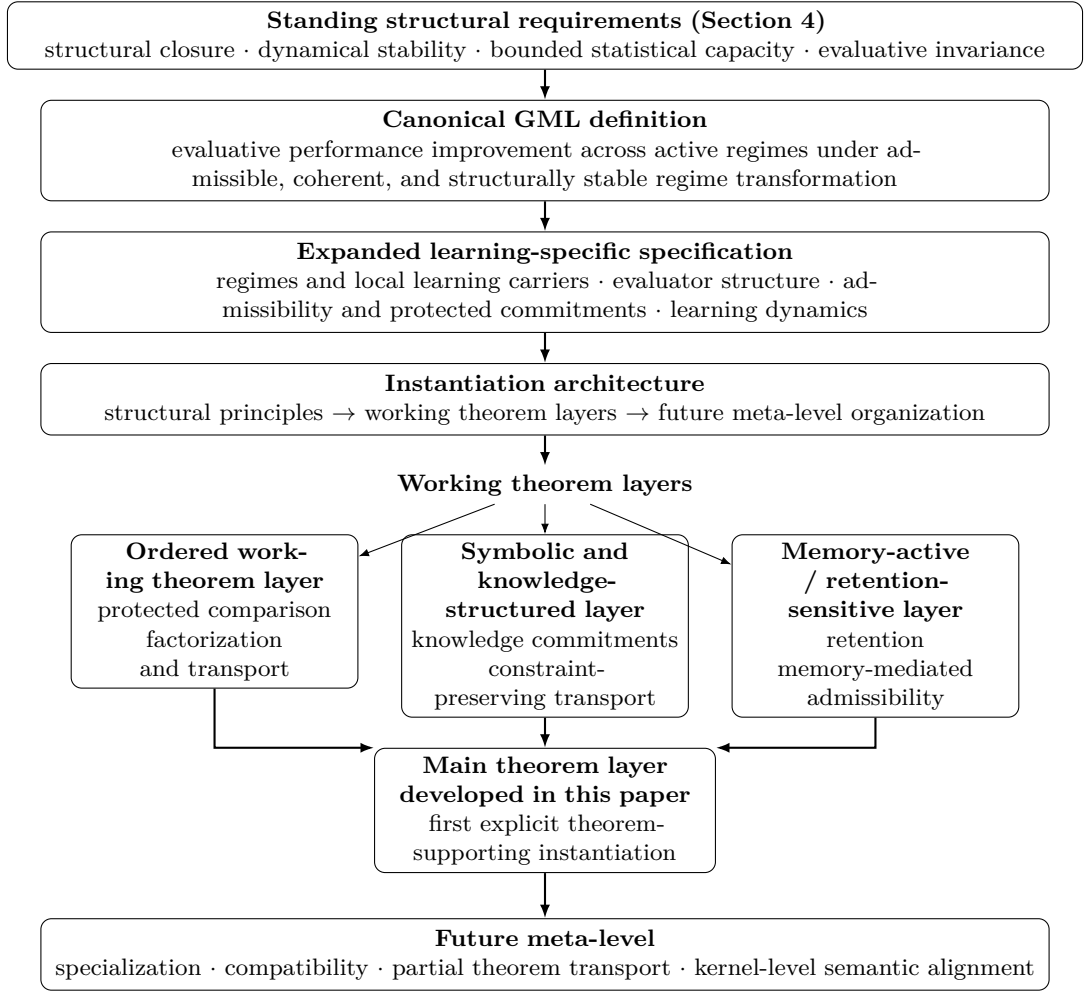

\section{First Theorem-Supporting Instantiation Developed in This Paper}
\label{sec:first-working-layer}

This section develops the first explicit theorem-supporting instantiation of GML in three coordinated blocks. Section~\ref{sec:reduction} recovers Mitchell-style fixed-ontology learning as a boundary-case degeneration and identifies the structural obstructions to faithful fixed-ontology reduction. Section~\ref{sec:stability} develops the protected stability template, the memory-mediated admissibility results, and the explicit convex witness. Section~\ref{sec:representation} establishes the morphism, factorization, composition, and synthesis results. All theorematic developments are stated under the four standing structural requirements of Section~\ref{sec:requirements}.

\subsection{Reduction to Classical Learning}
\label{sec:reduction}

\begin{assumption}[Classical degeneration conditions]
\label{ass:classical-degeneration}
Assume that:
\begin{enumerate}[leftmargin=*,itemsep=2pt]
    \item $\Rset=\{r_0\}$ is a singleton;
    \item $\Tcal=\{\mathrm{id}_{r_0}\}$ contains only the identity transformation;
    \item memory is semantically inert, in the sense that $\M_{r_0}$ carries no protected retention or admissibility-relevant obligation;
    \item the evaluator family reduces to one fixed scalar evaluator
    \[
    V_{r_0}=P;
    \]
    \item the protected core coincides with the evaluator itself,
    \[
    \Phi(V_{r_0})=V_{r_0};
    \]
    \item the local instance corresponds to one fixed task family $T$ and one fixed experience source $E$.
\end{enumerate}
\end{assumption}

This result is included as a boundary-case consistency check. Its role is to verify that when all genuinely regime-varying degrees of freedom are explicitly collapsed, the framework degenerates to the classical fixed-regime Mitchell ontology.

\begin{proposition}[Boundary-case degeneration to the fixed-regime Mitchell setting]
\label{prop:reduction}
Under Assumption~\ref{ass:classical-degeneration}, GML reduces to the statement that a program learns from experience $E$ with respect to task family $T$ and performance measure $P$ if its performance on $T$, as measured by $P$, improves with $E$.
\end{proposition}

\begin{proof}
Under Assumption~\ref{ass:classical-degeneration}, every specifically GML-level degree of freedom collapses. There is only one regime, so no cross-regime transport question arises; the only candidate transformation is the identity, so admissibility adds no extra constraint; memory carries no protected obligation, so it contributes no independent semantic content; and the evaluator is fixed and scalar, so protected-core preservation is trivial. What remains is precisely a learner receiving experience $E$, acting on tasks $T$, and being judged by one performance functional $P$. Learning is therefore identical to improvement of $P$ on $T$ with experience $E$, which is the Mitchell criterion.
\end{proof}

\paragraph{Interpretation of the reduction theorem.}
The proposition should be read as a boundary-case sanity result, not as a deep theorem of independent analytical content. Its role is to show exact compatibility with the classical ontology when regime variability, evaluator transport, and memory-mediated admissibility collapse. GML therefore extends the classical framework without retroactively invalidating it on its legitimate fixed-regime domain.

\begin{proposition}[Sufficient conditions for faithful Mitchell-reducibility]
\label{prop:mitchell-reducibility}
Let $\mathfrak G$ be a canonical GML system. Suppose:
\begin{enumerate}[leftmargin=*,itemsep=2pt]
    \item all regimes belong to a single protected-equivalence class;
    \item every admissible transformation acts trivially on the protected core;
    \item memory carries no admissibility-relevant information beyond what can be absorbed into a fixed experience description;
    \item all regime-local evaluators admit one common fixed scalar representative without structural loss;
    \item cross-regime comparison coincides with ordinary comparison under that representative.
\end{enumerate}
Then the induced GML process is faithfully Mitchell-expressible.
\end{proposition}

\begin{proof}
If all regimes lie in one protected-equivalence class and every admissible transformation acts trivially on that class, then regime variation does not create a new protected semantic distinction. If memory carries no irreducible admissibility information, then history can be absorbed into the classical experience description without semantic loss. If, moreover, all local evaluators share one fixed scalar representative and ordinary comparison agrees with protected comparison under that representative, then the whole process can be represented by a single fixed ontology $(E,T,P)$ while preserving the semantic data required for faithful fixed-ontology representation. Hence the process is faithfully Mitchell-expressible.
\end{proof}

\paragraph{Why this proposition matters.}
The point is not merely to say that classical learning is contained in GML. The point is to state when a GML process can genuinely be collapsed back into Mitchell's ontology without structural loss. This prepares the logical counterpart of the separation theorem: some processes satisfy Proposition~\ref{prop:mitchell-reducibility}, while others do not under faithful fixed-ontology reduction.

\paragraph{Morphismic reading of the reduction.}
In the language of Section~\ref{sec:representation}, the reduction can also be read as a degeneration through a protected-faithful morphism from a richer regime graph to a one-node fixed-regime image. This confirms that Mitchell-style learning is best understood as a degenerate morphic image of GML, not as a competing ontology of equal generality.

\begin{proposition}[Structural obstructions to faithful fixed-ontology reduction]
\label{prop:mitchell-obstruction}
Let $\mathfrak G$ be a multi-regime GML system. Under the structural notion of learning continuity adopted in this framework, a faithful reduction of $\mathfrak G$ to a single fixed-ontology Mitchell tuple $(E,T,P)$ is obstructed---meaning that the reduction cannot preserve the admissibility structure of learning evolution without structural loss---whenever any of the following conditions holds:
\begin{enumerate}[leftmargin=*,itemsep=2pt]
    \item \emph{Admissibility-critical memory:} admissibility strictly depends on retained-competence obligations in $\M_r$, so that absorbing memory into a passive experience history $E$ erases its operational role as a strict transition gate;
    \item \emph{Quotient-restricted comparability:} cross-regime evaluators $V_r$ and $V_{r'}$ are mathematically commensurable only on the quotient space induced by $\sim_\Phi$, so that any universal scalar comparison through $P$ extends evaluation beyond the protected quotient and thereby creates comparisons that are undefined at the intended semantic level;
    \item \emph{Non-aggregable protected cores:} the protected core $\Phi(V_r)$ functions as a hard admissibility requirement ($\Gamma \in \{0,1\}$), so that aggregating it with local utility into a standard compensable scalar objective $P$ structurally permits protected-core violation whenever local empirical gain compensates for the penalty.
\end{enumerate}
\end{proposition}

\begin{proof}
A reduction to a fixed ontology $(E,T,P)$ is structurally faithful only if it preserves the admissibility structure of learning evolution, neither prohibiting valid continuations nor authorizing transitions that violate protected semantic continuity. A purely static protected constraint may of course be encoded inside a restricted hypothesis class or a constrained ERM formulation. The present obstruction does not deny that possibility. It concerns the multi-regime case in which admissibility depends on typed transport, protected comparison, and memory-mediated obligations that evolve with the trajectory itself.
If condition (1) holds, mapping $\M_r$ to $E$ removes the active gating role played by admissibility, converting a strict boundary condition into contextual data. In particular, once admissibility depends on retained-competence obligations carried by memory, a faithful fixed-ontology reduction would have to encode the evolving history of protected commitments as part of the effective task state rather than merely as a static safe set.
If condition (2) holds, mapping a regime-indexed family $\{V_r\}$ to a universal evaluator $P$ destroys the quotient restriction, thereby extending comparison to states that do not share protected semantic content.
If condition (3) holds, replacing the strict admissibility certificate $\Gamma$ with a standard scalar penalty inside $P$ replaces a hard structural boundary with a compensable trade-off. A minimal analytic witness can be read on a scalarized objective of the form $P(s)=U(s)-\lambda\mathbf 1_{\{\text{protected core violated}\}}$ with finite $\lambda>0$. Whenever there exists a direction along which the local proxy gain in $U$ exceeds the finite penalty, a maximizing update can cross the protected boundary even though the corresponding GML transition is inadmissible. Preventing such violation uniformly requires either a singular penalty or an explicit projection/gating mechanism, which is precisely the structural role played here by admissibility certification and protected transport. In this sense, the obstruction is not merely that soft penalties may fail, but that dynamic admissibility with protected memory is not faithfully preserved by a fixed-ontology image unless the GML boundary objects are effectively rebuilt inside it.
In all three cases, mapping the GML process to a standard $(E,T,P)$ tuple incurs structural loss and therefore obstructs a faithful fixed-ontology reduction.
\end{proof}

\paragraph{Why the obstruction is operational, not merely linguistic.}
This proposition clarifies why a faithful fixed-ontology image is obstructed in the multi-regime case considered here. The obstruction does not arise merely because classical machine learning lacks the vocabulary of protected cores by stipulation. It is operational: under standard compensable fixed-objective reductions, hard admissibility constraints are converted into soft penalties, and memory-mediated admissibility is flattened into passive contextual data. Under the notion of faithful fixed-ontology reduction used here, such a reduction is obstructed because the fixed-ontology image structurally permits protected-core violation or loss of protected semantic continuity that the GML architecture explicitly prohibits through $\Gamma$ and $d_\Tcal$. Admissibility must therefore remain a first-class structural object rather than an artifact of a reward function.

We stress that the obstruction established here does not rest on the trivial claim that all hard constraints resist scalarization. Static constraints may indeed be encoded by constrained ERM or by restricting the admissible hypothesis class in a Mitchell-style setting. The stronger obstruction arises when admissibility depends on an actively transported memory state carrying retained competences or protected obligations inherited from prior regimes. In that case, a faithful fixed-ontology reduction would have to re-encode the trajectory-dependent semantic trace of past protected commitments inside a nominally static ontology. The point is therefore not merely that soft penalties may fail, but that dynamic admissibility with protected memory is not faithfully preserved by a fixed-ontology image unless the relevant GML boundary objects are effectively rebuilt inside it.

\paragraph{Why state-space enlargement does not dissolve the obstruction.}
A common objection is that a sufficiently enlarged Mitchell state could simply encode memory, evaluator metadata, and regime tags, thereby simulating any GML process inside a fixed ontology. The present obstruction result is not directed against such bare syntactic simulation. It is directed against faithful fixed-ontology reduction without structural loss. Enlarging the state does not by itself preserve admissibility as a first-class gate, quotient-restricted comparability as a protected semantic restriction, or the non-aggregability of hard protected cores. In particular, once admissibility is mediated by retained-competence memory, a fixed-state encoding must carry forward the dynamically evolving trace of protected commitments in order to reproduce the same legality structure. If an enlarged fixed-state encoding reproduces those boundary objects explicitly, then it has re-encoded the GML structure rather than reduced it away. The obstruction therefore concerns semantic preservation, not the mere possibility of storing more symbols in a larger state description.

\paragraph{What the reduction layer already fixes.}
The reduction and obstruction results already determine two indispensable parts of the present theory. First, they recover Mitchell-style fixed-ontology learning as a genuine boundary-case degeneration of GML rather than as a competing ontology of equal generality. Second, they isolate the first irreducibly multi-regime zone of the theory by showing exactly when faithful fixed-ontology collapse remains possible and when it fails without structural loss. The broader first theorem-supporting consolidation of the present layer is established only after the later stability, persistence, morphism, transport, and composition results are added.

\subsection{Stability, Memory, and Regime Transitions}
\label{sec:stability}

\paragraph{Protected stability template and formal development.}
The results below develop the stability layer of the first theorem-supporting instantiation. Stability is not an optional refinement: it is one of the conditions under which learning remains theoretically tenable across active regimes once evaluator transport, protected guarantees, and memory-mediated admissibility become part of the learning object. In GML, stability is downstream from a more basic chain:
\[
\text{protected-core preservation}
\;\Longrightarrow\;
\text{protected equivalence}
\;\Longrightarrow\;
\text{meaningful evaluator transport}  
\]
\[
\;\Longrightarrow\;
\text{persistence of learning guarantee}
\;\Longrightarrow\;
\text{stability at the correct abstraction level}.
\]
This point is essential. The present section does not aim to prove a universal generalization bound. Its aim is to isolate the stability template under which admissible regime variation need not destroy the identity of the learning problem nor the meaning of the guarantees attached to it.

\begin{definition}[Protected-level stability]
A GML trajectory
\[
(r_t,s_t,m_t)_{t\geq 0}
\]
is \emph{protected-level stable} if there exists a protected-level discrepancy functional
\[
W_t = W(r_t,s_t,m_t;\Phi(V_{r_t}))
\]
such that along admissible transitions, the quantity $W_t$ remains controlled in a way that preserves the semantic comparability of the learning guarantees attached to the protected core.
\end{definition}

\paragraph{Theoretical role of protected stability.}
The theory does not require raw state trajectories to remain invariant under regime change. That would be too strong and often meaningless. What must remain stable is the learning process at the level where its identity is protected: the level of the protected core, the admissibility certificate, and the guarantees that survive transport.
\begin{theorem}[Template Theorem]
\label{thm:protected-stability}
\textit{Protected-level discrepancy bounds under admissible transport.}

Let $(r_t,s_t,m_t)_{t\ge 0}$ be a GML trajectory governed by admissible regime transitions, so that $\Gamma(r_t,r_{t+1},\tau_t;V_{r_t},V_{r_{t+1}})=1$ for all $t$. Assume there exists a nonnegative protected-level discrepancy functional
\[
W_t = W(r_t,s_t,m_t;\Phi(V_{r_t}))
\]
satisfying the local drift condition
\[
\mathbb{E}[W_{t+1}\mid \mathcal F_t]
\le
(1-\alpha)W_t+\delta+\beta\, d_{\Tcal}(r_t,r_{t+1}),
\]
for constants $\alpha\in(0,1]$, $\delta\ge 0$, and $\beta\ge 0$, where $\mathcal F_t$ is the filtration of the learning process up to time $t$, and where $d_{\Tcal}(r_t,r_{t+1})<\infty$ is the certified structural cost of the admissible transition. Whenever the regime sequence is stochastic, the cost term $d_{\Tcal}(r_t,r_{t+1})$ is assumed $\mathcal F_t$-measurable. Then for every $n\ge 1$,
\[
\mathbb E[W_n]
\le
(1-\alpha)^n\mathbb E[W_0]
+
\frac{\delta}{\alpha}
+
\beta\sum_{k=0}^{n-1}(1-\alpha)^{n-1-k}d_{\Tcal}(r_k,r_{k+1}).
\]
Furthermore, the long-run protected-level behavior depends on the asymptotic regime-variation profile as follows:
\begin{enumerate}[leftmargin=*,itemsep=2pt]
    \item \emph{Ultimate boundedness:} if the admissible regime costs are uniformly bounded, say $d_{\Tcal}(r_t,r_{t+1})\le \bar d$ for all $t$, then
    \[
    \limsup_{n\to\infty}\mathbb E[W_n]
    \le
    \frac{\delta+\beta\bar d}{\alpha}.
    \]
    \item \emph{Asymptotic recovery to the noise floor:} if
    \[
    d_{\Tcal}(r_t,r_{t+1})\to 0,
    \]
    then the transport contribution vanishes asymptotically and
    \[
    \limsup_{n\to\infty}\mathbb E[W_n]
    \le
    \frac{\delta}{\alpha}.
    \]
    \item \emph{Exact convergence:} if, in addition, $\delta=0$ and
    \[
    d_{\Tcal}(r_t,r_{t+1})\to 0,
    \]
    then
    \[
    \lim_{n\to\infty}\mathbb E[W_n]=0.
    \]
\end{enumerate}
\end{theorem}

\begin{proof}
Taking total expectations in the local drift inequality gives
\[
\mathbb E[W_{t+1}]
\le
(1-\alpha)\mathbb E[W_t]+\delta+\beta\,d_{\Tcal}(r_t,r_{t+1}).
\]
Unrolling this linear recurrence from $t=0$ to $t=n-1$ yields
\[
\mathbb E[W_n]
\le
(1-\alpha)^n\mathbb E[W_0]
+
\sum_{k=0}^{n-1}(1-\alpha)^{n-1-k}\delta
+
\beta\sum_{k=0}^{n-1}(1-\alpha)^{n-1-k}d_{\Tcal}(r_k,r_{k+1}),
\]
and the geometric-series estimate
\[
\sum_{k=0}^{n-1}(1-\alpha)^{n-1-k}\delta \le \frac{\delta}{\alpha}
\]
gives the stated finite-horizon bound. If $d_{\Tcal}(r_t,r_{t+1})\le \bar d$, then the transport term is bounded above by
\[
\beta\bar d\sum_{k=0}^{n-1}(1-\alpha)^{n-1-k}\le \frac{\beta\bar d}{\alpha},
\]
which yields the first asymptotic regime. If instead $d_{\Tcal}(r_t,r_{t+1})\to 0$, then the discrete convolution of the vanishing sequence $\big(d_{\Tcal}(r_t,r_{t+1})\big)_{t\ge 0}$ with the strictly contractive geometric kernel $\big((1-\alpha)^k\big)_{k\ge 0}$ also vanishes, which yields the second regime. The third follows immediately when $\delta=0$, since then
\[
\limsup_{n\to\infty}\mathbb E[W_n]\le 0
\]
and $\mathbb E[W_n]\ge 0$ by nonnegativity.

This result is explicitly a structural template rather than a standalone convergence theorem. Its role is to fix the exact control schema under which protected discrepancy can be propagated through admissible regime variation. The burden of constructing $W_t$ and verifying the local contraction parameters belongs to the specific working theorem layers, as illustrated constructively in Proposition~\ref{prop:two-regime-witness}. The three asymptotic regimes above show how long-run protected behavior depends on the regime-variation profile: bounded cost yields ultimate boundedness, vanishing cost yields recovery to the noise floor, and vanishing cost with $\delta=0$ yields exact convergence.
\end{proof}

\paragraph{Relation to standard drift analysis.}
The algebraic unrolling in the protected-stability template is formally standard and may be read as a discrete-time Foster--Lyapunov type estimate. The specifically GML-level content lies in the structural conditions under which the discrepancy functional, admissibility gate, and transported complexity control make such a template non-vacuous across regime changes.

\paragraph{High-probability concentration.}
While Theorem~\ref{thm:protected-stability} is stated in expectation in order to isolate the clean structural template, the formulation through the filtration $\mathcal F_t$ is compatible with standard high-probability refinements. In subclasses where the single-step fluctuation terms $W_{t+1}-\mathbb E[W_{t+1}\mid\mathcal F_t]$ are uniformly bounded or sub-Gaussian, martingale concentration tools such as Azuma--Hoeffding type inequalities can be applied to the drift process. This does not yield a separate theorem here, but it shows that the expectation-level stability template is naturally aligned with finite-horizon $(\epsilon,\delta)$-style PAC relaxations of admissibility in stochastic theorem-layer instantiations.

\paragraph{Status of the regime-variation cost $d_{\Tcal}$.}
The theory does not require one universal metric for regime variation. At the level of the present theorem layer, $d_{\Tcal}$ should be read as a typed cost of admissible regime change whose role is made explicit in the drift control. Four families are especially natural in the present framework: (i) \emph{distributional regime costs}, when regime change is primarily a shift in data law or comparator law; (ii) \emph{structural regime costs}, when the typed state/interface transport itself changes nontrivially; (iii) \emph{memory-sensitive costs}, when admissibility depends on retained-competence or protected-memory updates; and (iv) \emph{evaluator-aware costs}, when the main burden lies in transporting the local evaluator while preserving the protected core. Hybrid costs combining several of these components are allowed. What matters here is not commitment to one canonical metric, but the requirement that the cost be explicit, typed, and semantically compatible with the protected evaluator structure used in the theorem layer. In the concrete two-regime witness below, and in the worked toy example of Section~\ref{sec:examples}, $d_{\Tcal}$ is instantiated by an anchor-shift-controlled transport term within a convex quadratic subclass, precisely to show how the protected stability template specializes once the source of regime variation is made explicit. The paper therefore does not treat $d_{\Tcal}$ as a canonical divergence attached once and for all to every learning problem. It should be read as a theorem-layer-certified cost of admissible regime variation which, in controlled subclasses, may be instantiated or bounded by more familiar shift measures. More generally, the structural component of $d_{\Tcal}$ may also absorb observation/action interface mismatch through the continuity modulus of $(\tau_O,\tau_A)$, so that the interface carriers $\Obs_r$ and $\Act_r$ play a mathematically active role whenever regime change alters how the learner observes or acts. Throughout the present theorem layer, admissibility and regime cost are linked rather than independent: finite regime cost presupposes admissibility certification, whereas inadmissible transitions incur infinite semantic cost and are excluded from protected learning continuation. Consequently, an active GML learning trajectory is strictly confined to the finite-cost subgraph of $\Tcal$; attempting an inadmissible transition topologically terminates the current protected learning problem.

\begin{proposition}[Explicit Witness for the Protected Stability Template]
\label{prop:two-regime-witness}
The protected stability template defined in Theorem~\ref{thm:protected-stability} is non-vacuous. Specifically, within a convex anchor-based subclass of commensurable regimes, one can construct a discrepancy functional $W_t$ such that:
\begin{enumerate}[leftmargin=*,itemsep=2pt]
    \item $W_t$ satisfies the contractive drift $\alpha > 0$ whenever the local update rule is strictly contractive toward the regime-local semantic anchor;
    \item in the present construction, the transport overhead $\beta d_\Tcal(r,r')$ is explicitly controlled by the squared Euclidean shift between the anchors $\mu(r)$ and $\mu(r')$.
\end{enumerate}
This witness establishes that the protected stability template admits an explicit realization within a standard convex-quadratic learning setting.
\end{proposition}

\begin{proof}
The proof is by instantiation within a controlled subclass. Consider a setting where the protected core $\Phi$ maps to a convex parameter domain $\mathcal S \subseteq \mathbb R^d$. For each regime $r$, let $\mu(r)\in\mathcal S$ be a stationary semantic anchor. We define the discrepancy as the squared distance to this anchor:
\[
W_t(r)=\|s_t-\mu(r)\|^2.
\]

\emph{Local Drift.}
Let the local update $s_{t+1}=f(s_t,r)$ be strictly contractive toward $\mu(r)$, such that
\[
\|f(s_t,r)-\mu(r)\|^2 \le (1-\alpha)\|s_t-\mu(r)\|^2
\]
for some $\alpha\in(0,1]$. This satisfies the drift condition of the template with $\delta=0$.

\emph{Transport Overhead.}
Upon a regime transition $\tau:r\to r'$, the discrepancy relative to the new anchor is
\[
W_{t+1}(r')=\|s_{t+1}-\mu(r')\|^2.
\]
By the triangle inequality and Young's inequality, for any $\epsilon>0$,
\[
\|s_{t+1}-\mu(r')\|^2
\le
(1+\epsilon)\|s_{t+1}-\mu(r)\|^2
+
(1+1/\epsilon)\|\mu(r)-\mu(r')\|^2.
\]
Since $\alpha\in(0,1]$, one has $\alpha/(1-\alpha)>0$, so any $\epsilon\in(0,\,\alpha/(1-\alpha))$ satisfies $(1+\epsilon)(1-\alpha)<1$. Fixing any such $\epsilon$ and instantiating $\beta d_\Tcal(r,r')$ as the shift-dependent term $(1+1/\epsilon)\|\mu(r)-\mu(r')\|^2$, the construction recovers the structural form required by the transport bound in Theorem~\ref{thm:protected-stability}. This explicit case demonstrates that the template is consistent with standard quadratic Lyapunov analyses in convex optimization.

\emph{Capacity control.}
The convex subclass used here satisfies the bounded-capacity standing requirement of Section~\ref{sec:requirements} with $L_\tau=1$ and $\mathcal C_{\mathrm{new}}(r\to r')=0$. The state transport $\tau_S$ is a gradient step on a fixed parameter domain $\mathcal S\subseteq\mathbb R^d$, so the covering numbers satisfy
\[
\log\mathcal N(\mathcal S,\epsilon,\|\cdot\|)\le d\log\!\left(1+\frac{2\,\mathrm{diam}(\mathcal S)}{\epsilon}\right)
\]
independently of the regime transition. No genuinely new independent degrees of freedom are introduced by the anchor shift; only the reference point $\mu(r)$ changes. The innovation term $\mathcal C_{\mathrm{new}}(r\to r')$ is therefore zero in this subclass, and the capacity-admissibility condition of Section~\ref{sec:requirements} is satisfied by construction. This confirms that the witness is compatible with all four standing structural requirements simultaneously.
\end{proof}

\paragraph{Status of the witness.}
This concrete two-regime witness is not presented as a canonical algorithmic solution. Its role is narrower and more important for the present manuscript: it is a minimal end-to-end theorem-layer realization showing that admissibility certificates, regime-variation costs, memory updates, and protected-level stability can already be made to interact within one explicit typed instance of GML. In particular, it shows how a regime-specific contractive update on the protected quotient and a bounded memory defect feed directly into the drift control of Theorem~\ref{thm:protected-stability}. A concrete end-to-end realization of this construction~--- in which $\Phi$, $\Gamma$, and $d_{\Tcal}$ are made fully explicit on a two-regime linear regression instance~--- is worked out in Section~\ref{subsec:toy}.

\paragraph{Analytical reading of the witness.}
The witness acquires a concrete analytical interpretation in this subclass. The contraction parameter $\alpha\in(0,1]$ is the local contraction rate toward the regime-local anchor $\mu(r)$; the splitting constant $\epsilon\in(0,\alpha/(1-\alpha))$ balances source-drift and transport overhead; and
\[
\beta\,d_{\Tcal}(r,r')=(1+1/\epsilon)\|\mu(r)-\mu(r')\|^2
\]
is the transport overhead controlled by the squared anchor displacement. The abstract parameters of Theorem~\ref{thm:protected-stability} are therefore not purely formal: they carry concrete geometric meanings in the convex quadratic subclass, and $\epsilon$ exists whenever $\alpha>0$. In richer instantiations, $d_{\Tcal}$ may absorb evaluator-transport or memory-degradation terms; the present paper leaves those components to future subclasses.

\begin{definition}[Persistence of guarantee under stability]
A learning guarantee is \emph{stably persistent} along an admissible GML trajectory if:
\begin{enumerate}[leftmargin=*,itemsep=2pt]
    \item the guarantee is formulated at the protected-core level;
    \item each transition in the trajectory is admissible;
    \item the corresponding protected-level discrepancy remains controlled in the sense of Theorem~\ref{thm:protected-stability}.
\end{enumerate}
\end{definition}

\begin{proposition}[Persistence of Stability under Admissible Transport]
\label{prop:stable-persistence}
A protected-level stability guarantee, formulated as a contractive bound on the discrepancy functional $W_t$ in regime $r$, persists under admissible transport. Specifically, the convergence-control properties of the learning process do not collapse under transport; they persist in the target regime $r'$ in a transport-adjusted form, with a bounded residual $\mathcal{R}(\tau)=O(d_\Tcal(r,r'))$ controlled by the admissible transport cost.
\end{proposition}

\begin{proof}
The proof follows from the structural invariance of the protected core combined with the additive role of the transport term in the stability template.

By the persistence of the protected core (Observation~\ref{obs:coherence}, among the immediate structural consequences of admissibility), an admissible transition ($\Gamma=1$) preserves the semantic commitments that define the protected core. Since the stability guarantee is formulated relative to the discrepancy $W_t$, and $W_t$ is anchored in this invariant core, the guarantee remains semantically meaningful and controlled in the target regime $r'$.

The transition $\tau$ introduces a structural shift which, according to the Protected Stability Template (Theorem~\ref{thm:protected-stability}), manifests as an additive transport term $\beta d_\Tcal(r,r')$. Because $d_\Tcal$ is finite for any admissible transition, this term induces only a bounded shift in the asymptotic discrepancy floor, without collapsing the protected-level stability control. Consequently, admissible transport preserves the stability bound in a transport-adjusted form, with degradation bounded by the regime-shift residual. This shows that protected-level stability control persists along admissible GML trajectories.
\end{proof}

\paragraph{Concrete significance for learning theory.}
This proposition clarifies the distinction between ordinary adaptation and semantically disciplined learning evolution. A learner may continue to perform well numerically while losing the identity of the protected guarantee that justified its earlier performance. GML blocks that ambiguity by tying guarantee persistence to admissibility and protected stability rather than to raw score preservation alone.

\paragraph{Memory, forgetting, and admissible retention.}
Memory is not an application-side feature: it enters the structural specification of the learning process itself so that forgetting and retention can be judged as admissible or inadmissible components of learning continuity. Memory may carry protected obligations from past regimes and therefore affect which future transitions are admissible.

\begin{definition}[Admissible retention constraint]
An \emph{admissible retention constraint} is a protected-core condition requiring that some certified competence, safety property, or comparison relation inherited from a previous regime remain above a specified protected threshold after regime transition.
\end{definition}

\begin{proposition}[Memory-mediated admissibility]
\label{prop:memory-admissibility}
There exist GML systems in which two candidate transitions with identical local proxy gain differ in admissibility because one preserves a memory-mediated retention constraint and the other violates it.
\end{proposition}

\begin{proof}
Consider a continual learner whose protected core contains a minimum retained-competence floor inherited from earlier regimes. Let two candidate updates produce the same value of the new-regime proxy score. If one update preserves the retained-competence floor and the other violates it, then the first update satisfies the admissibility clauses while the second fails them. Hence admissibility distinguishes the two transitions even though local proxy gain does not.
\end{proof}

\paragraph{Interpretive consequence.}
This is one of the simplest points at which the classical fixed-task ontology becomes too coarse. A theory that records only local gain cannot distinguish between legitimate adaptation and semantically destructive adaptation when retained competence is part of the learning problem itself.

\begin{proposition}[\footnotesize Structural necessity of internalized protected continuity, admissibility, and stability]
\label{prop:necessity-internalized}
Consider a regime-varying learning process for which the following are required:
\begin{enumerate}[leftmargin=*,itemsep=2pt]
    \item the learning problem remain identifiable across regime change at a protected evaluative level;
    \item cross-regime guarantees remain meaningful rather than merely local;
    \item semantically legitimate adaptation be distinguishable from semantically destructive adaptation.
\end{enumerate}
Then any faithful learning-theoretic representation of that process must internalize:
\begin{enumerate}[leftmargin=*,itemsep=2pt]
    \item a protected evaluative core;
    \item an admissibility criterion on regime transport;
    \item a stability notion controlling protected-level discrepancy.
\end{enumerate}
In particular, none of these may remain a merely external side condition without losing at least one of the three requirements above.
\end{proposition}

\begin{proof}
If no protected evaluative core is internalized, then there is no structure left to identify the same learning problem across regimes at the level required by the abstract coherence observation (Observation~\ref{obs:coherence}); protected continuity becomes undefined rather than preserved. If admissibility is not internalized, Proposition~\ref{prop:memory-admissibility} yields candidate transitions with identical local proxy gain but different semantic status, so the theory can no longer distinguish legitimate adaptation from semantically destructive adaptation. If stability is not internalized, then Theorem~\ref{thm:protected-stability} and Proposition~\ref{prop:stable-persistence} no longer provide any controlled route along which guarantees remain meaningful beyond a local step. Therefore any faithful representation satisfying the three requirements must internalize protected continuity, admissibility, and stability as parts of the learning object itself.
\end{proof}

\paragraph{Status of the necessity claim.}
The necessity established in Proposition~\ref{prop:necessity-internalized} is not intended as a universal meta-theorem about every conceivable theory of learning. It is internal to the present requirements of well-posed learning under admissible regime variation: once protected problem identity, cross-regime guarantee persistence, and the distinction between legitimate and semantically destructive adaptation are all demanded, protected continuity, admissibility, and stability can no longer remain external side conditions.

\subsection{Morphisms, Representation, Composition, and Invariance}
\label{sec:representation}

\subsubsection{Morphisms between GML systems}
The morphism results are included for more than abstract elegance: they show that the framework supports disciplined comparison, factorization, and transfer between learning systems rather than mere descriptive variability.

\begin{definition}[GML morphism]
Let
\[
\mathfrak G=(\Rset,\{\State_r\},\{\M_r\},\{\Obs_r\},\{\Act_r\},\{V_r\},\Tcal,\Phi,\Gamma,\sim_\Phi,\Upd)
\]
and
\[
\mathfrak G'=(\Rset',\{\State'_{r'}\},\{\M'_{r'}\},\{\Obs'_{r'}\},\{\Act'_{r'}\},\{V'_{r'}\},\Tcal',\Phi',\Gamma',\sim_{\Phi'},\Upd')
\]
be two canonical GML systems. A \emph{GML morphism}
\[
F:\mathfrak G\to \mathfrak G'
\]
consists of:
\begin{enumerate}[leftmargin=*,itemsep=2pt]
    \item a regime map
    \[
    F_R:\Rset\to\Rset';
    \]
    \item for each regime $r\in\Rset$, typed component maps
    \[
    F_S^r:\State_r\to \State'_{F_R(r)},\qquad
    F_M^r:\M_r\to \M'_{F_R(r)},\qquad
    F_O^r:\Obs_r\to \Obs'_{F_R(r)},\qquad
    F_A^r:\Act_r\to \Act'_{F_R(r)};
    \]
    \item an evaluator transport rule sending each $V_r$ to an evaluator compatible with $V'_{F_R(r)}$;
    \item a transformation map sending each admissible arrow $\tau:r\rightsquigarrow r'$ in $\Tcal$ to an admissible arrow
    \[
    F_T(\tau):F_R(r)\rightsquigarrow F_R(r')
    \]
    in $\Tcal'$;
    \item preservation of the protected core and admissibility discipline, meaning:
    \begin{enumerate}[leftmargin=*,itemsep=2pt]
        \item if $\Phi(V_r)\sim_\Phi \Phi(V_{r'})$, then
        \[
        \Phi'(F(V_r))\sim_{\Phi'} \Phi'(F(V_{r'}));
        \]
        \item if $\Gamma(r,r',\tau;V_r,V_{r'})=1$, then
        \[
        \Gamma'(F_R(r),F_R(r'),F_T(\tau);F(V_r),F(V_{r'}))=1.
        \]
    \end{enumerate}
\end{enumerate}
\end{definition}

\paragraph{What a GML morphism means.}
A morphism does not merely translate notation. It transports a learning system into another one while preserving the structural meaning of regime, memory, evaluator discipline, and admissibility. This is the correct notion for comparing instantiations of GML without collapsing them into raw implementation details.

\paragraph{Why morphisms are useful.}
Morphisms serve at least four purposes:
\begin{enumerate}[leftmargin=*,itemsep=2pt]
    \item \emph{reduction}: relating a richer GML system to a simpler or degenerate one;
    \item \emph{simulation}: showing that one regime structure is representable inside another;
    \item \emph{quotienting}: merging regimes while preserving the protected semantics;
    \item \emph{transfer of guarantees}: transporting admissibility and protected-core properties across implementations.
\end{enumerate}

\begin{definition}[Protected-faithful morphism]
A GML morphism is \emph{protected-faithful} if it reflects as well as preserves protected equivalence and admissibility, i.e., no protected distinction or admissibility distinction is collapsed by the map unless such collapse is explicitly certified by the target system.
\end{definition}

\begin{proposition}[Morphisms preserve admissible trajectories]
\label{prop:morphism-trajectories}
Let
\[
F:\mathfrak G\to\mathfrak G'
\]
be a GML morphism. Then the image of any admissible trajectory in $\mathfrak G$ is an admissible trajectory in $\mathfrak G'$.
\end{proposition}

\begin{proof}
The admissibility-preservation clause in the definition of GML morphism requires that if $\Gamma(r,r',\tau;V_r,V_{r'})=1$, then $\Gamma'(F_R(r),F_R(r'),F_T(\tau);F(V_r),F(V_{r'}))=1$. Applying this clause to each primitive arrow of an admissible source trajectory produces an admissible image arrow at every step. Since the component maps preserve the corresponding state, memory, and interface transports, the full image trajectory satisfies the admissibility discipline of the target system.
\end{proof}

\begin{proposition}[Morphisms preserve protected guarantees]
\label{prop:morphism-guarantees}
If a learning guarantee is formulated on the protected core in $\mathfrak G$ and transported through a protected-faithful morphism
\[
F:\mathfrak G\to\mathfrak G',
\]
then the corresponding guarantee remains meaningful in the target system.
\end{proposition}

\begin{proof}
A protected-core guarantee is formulated on the quotient structure induced by protected equivalence together with the admissibility discipline governing legitimate transport. A protected-faithful morphism preserves that structure and does not collapse distinctions unless the collapse is itself certified in the target. Therefore the transported guarantee remains meaningful in the target system.
\end{proof}

\begin{proposition}[Mitchell-style learning as a degenerate morphic image]
\label{prop:mitchell-morphic-image}
Let $\mathfrak G$ be a GML system whose admissible regime graph collapses under a protected-faithful morphism to a single regime, whose admissibility certificate becomes trivial on the image, and whose protected core coincides with one fixed scalar performance criterion. Then the image system is a Mitchell-style learning system in the sense that its learning description is exhausted by one fixed experience/task/performance ontology. In this precise sense, Mitchell-style learning appears as a degenerate morphic image of GML rather than as an ontologically independent alternative.
\end{proposition}

\begin{proof}
Collapse the regime graph to one node, transport all admissible transitions to the identity of that node, trivialize $\Gamma$ on the image, and identify the protected core with the fixed scalar performance criterion preserved by the morphism. The resulting collapse is protected-faithful relative to the trivial equivalence class of the target image: because the Mitchell-style target carries one fixed scalar evaluator, protected equivalence there coincides with strict identity of that evaluator and no further protected distinctions remain to be preserved. What remains is exactly a learning description in which one stable experience/task/performance semantics suffices, i.e. the Mitchell-style case. The construction does not refute Mitchell's ontology; it shows that fixed-ontology learning is recovered as a morphically degenerate image of the broader GML object.
\end{proof}

\subsubsection{Protected-core factorization and representation}

\begin{proposition}[Protected-core factorization of cross-regime evaluation]
\label{prop:protected-factorization}
Suppose each regime-local evaluator admits a decomposition
\[
V_r = G_r(U_r,C_r),
\]
where:
\begin{enumerate}[leftmargin=*,itemsep=2pt]
    \item $U_r$ is a regime-local utility component taking values either in a scalar codomain or in a partially ordered vector space equipped with a specified order, for example a Pareto order on a common ordered codomain, so that vector-valued utilities are compared only relative to an explicitly specified partial order rather than a hidden total scalarization;
    \item $C_r$ is a protected-core component;
    \item for every admissible transformation, $C_r$ is preserved up to protected equivalence;
    \item $G_r$ is monotone in its utility coordinate once the protected-equivalence class of $C_r$ is fixed, where monotonicity is understood relative to the chosen scalar order or partial order on the utility space.
\end{enumerate}
Then cross-regime comparison descends to comparison of the utility coordinates on the quotient induced by protected equivalence.
\end{proposition}

\begin{proof}
Along any admissible chain, the protected-core components $C_r$ and $C_{r'}$ belong to one protected-equivalence class by assumption. Consequently the part of the evaluator depending on $C_r$ is fixed at the semantic level relevant for comparison. Once that quotient class is held constant, the evaluator ordering is governed by the utility coordinate through the monotonicity of $G_r$ in its first argument. Hence protected-level cross-regime comparison descends to comparison of the utility components on the quotient.
\end{proof}

\paragraph{Structural dividend of the factorization.}
This proposition explains why protected equivalence is not merely a philosophical convenience. It gives a concrete mathematical dividend: evaluator plurality does not destroy comparability provided the protected component is stabilized and comparison is taken at the correct quotient level.
When the utility component is vector-valued, monotonicity is understood relative to a specified partial order on the utility codomain, for instance a Pareto-type order on a common ordered vector space. The present factorization therefore does not assume that all meaningful utility comparisons are scalar or totally ordered.

\begin{proposition}[Invariance of protected semantics under protected-faithful morphisms]
\label{prop:morphism-invariance}
Let
\[
F:\mathfrak G\to\mathfrak G'
\]
be a protected-faithful morphism. Then:
\begin{enumerate}[leftmargin=*,itemsep=2pt]
    \item protected equivalence classes are preserved and reflected by $F$ up to explicit certification in the target;
    \item admissible transition structure is preserved by $F$;
    \item any protected-core guarantee transportable in $\mathfrak G$ remains transportable in $\mathfrak G'$.
\end{enumerate}
\end{proposition}

\begin{proof}
Protected-faithfulness requires both preservation and reflection of the protected-equivalence and admissibility structure, except where explicit certification in the target authorizes a collapse. This yields the first two claims. The third then follows because protected-core guarantees are transported exactly through those two structures: if equivalence classes and admissible arrows are preserved, so is the semantic route along which guarantees travel.
\end{proof}

\subsubsection{Composition, synthesis, and sufficiency}

\begin{theorem}[Composition of admissible learning segments]
\label{thm:composition}
Let
\[
r_0 \xrightarrow{\tau_0} r_1 \xrightarrow{\tau_1} \cdots \xrightarrow{\tau_{n-1}} r_n
\]
be a composable chain of regime transitions such that:
\begin{enumerate}[leftmargin=*,itemsep=2pt]
    \item each $\tau_k$ is admissible;
    \item each local segment admits certified evaluative improvement;
    \item protected-core preservation is certified at each step;
    \item each segment is protected-level stable in the sense of Theorem~\ref{thm:protected-stability};
    \item the certification rules are closed under the required concatenations.
\end{enumerate}
Then the whole chain defines a well-formed global GML learning segment with persistent protected semantics and transportable guarantees.
\end{theorem}

\begin{proof}
We proceed by induction on $n$.

\emph{Base case ($n=1$).} For a single admissible segment $r_0 \xrightarrow{\tau_0} r_1$, admissibility gives a well-typed certified transport, condition~2 gives local evaluative improvement, condition~3 certifies protected-core preservation, and condition~4 gives protected-level stability in the sense of Theorem~\ref{thm:protected-stability}. Proposition~\ref{prop:stable-persistence} then yields that the stability guarantee persists at $r_1$ in transport-adjusted form. Hence the one-step chain is a well-formed GML learning segment with persistent protected semantics.

\emph{Inductive step.} Suppose the claim holds for any composable chain of length $k$ ending at regime $r_k$. By admissibility and the closure hypothesis, concatenating the next arrow $\tau_k:r_k\rightsquigarrow r_{k+1}$ yields a new well-typed admissible segment. Protected-core preservation ensures that the semantic identity of the learning problem is maintained at the next step, and protected-level stability together with Proposition~\ref{prop:stable-persistence} ensures that the stability control persists at $r_{k+1}$ with a bounded transport residual. Since condition~2 certifies evaluative improvement at each step, the extended chain of length $k+1$ remains a well-formed global GML learning segment. The result follows by induction.
\end{proof}

\begin{corollary}[Finite-chain guarantee transport]
\label{cor:finite-chain-transport}
Under the assumptions of Theorem~\ref{thm:composition}, any learning guarantee formulated on the protected core at regime $r_0$ remains meaningful, literally or up to protected equivalence, throughout the whole finite chain.
\end{corollary}

\begin{proof}
Apply Proposition~\ref{prop:stable-persistence} to the first admissible segment. Theorem~\ref{thm:composition} shows that composition preserves the protected semantics and transport discipline needed to apply the same argument to the next segment. Iterating this reasoning over the whole finite chain yields the conclusion.
\end{proof}

\paragraph{Scope of the corollary.}
The conclusion depends on all five hypotheses of Theorem~\ref{thm:composition}. In particular, condition~4 requires that each segment admit a protected-level discrepancy functional satisfying the contractive drift assumptions of Theorem~\ref{thm:protected-stability}. The present paper verifies such a condition only in the convex quadratic subclass of Proposition~\ref{prop:two-regime-witness}; the corollary should therefore be read as a conditional transport principle rather than as a universal guarantee.

\paragraph{PAC composition remark.}
Under the probabilistic admissibility reading of $\Gamma$, the deterministic compositional statement above should be read through the filtration generated by the evolving learning process rather than through an independence assumption across regimes. Concretely, if the $k$-th transition satisfies a conditional admissibility guarantee of the form
\[
\Pr\big(\Gamma_k=1\mid \mathcal F_{k-1}\big)\ge 1-\delta_{\Gamma,k},
\]
then the probability that the whole length-$n$ chain remains admissible is controlled by the corresponding product of conditional success probabilities, and therefore in particular by the usual first-order bound obtained from the chain rule. The deterministic theorem used in the present paper is thus the exact-certification limit of a filtration-aware probabilistic admissibility composition principle rather than a denial of stochastic degradation under repeated transport.

\paragraph{Asymptotic viability remark.}
The stability theorem and its finite-chain corollary are intentionally stated on finite admissible horizons. For strictly lifelong or asymptotic stability, one needs additional viability conditions beyond those imposed in the present manuscript, for example summability of the cumulative regime-variation costs, eventual quiescence of regime severity, or a genuine consolidation mechanism on memory-side defects. In subclasses where the regime-variation cost locally dominates the metric on the protected quotient---for instance when there exists $\kappa>0$ such that
\[
d_{\mathcal E/\sim_\Phi}(\mu(r_t),\mu(r_{t+1}))\le \kappa\sqrt{d_{\Tcal}(r_t,r_{t+1})}
\]
along an admissible chain---the summability of $\sum_t d_{\Tcal}(r_t,r_{t+1})$ forces the anchor sequence $(\mu(r_t))_{t\ge 0}$ to be Cauchy in the protected quotient. Together with the completeness assumptions already stated for the quotient subclass, this yields a mathematically well-posed asymptotic protected limit. The current theorem layer therefore establishes controlled protected stability on admissible finite chains while also making explicit which additional domination and completeness hypotheses would be needed for a full asymptotic theory.

\begin{theorem}[Synthesis of the Present Theorem-Supporting Layer]
\label{thm:first-layer-synthesis}
Within the present working theorem layer, the canonical GML definition already supports a non-vacuous and structurally distinct theorem-supporting instantiation. More precisely:
\begin{enumerate}[leftmargin=*,itemsep=2pt]
    \item the fixed-ontology Mitchell case is recovered as a boundary-case degeneration by Proposition~\ref{prop:reduction};
    \item faithful Mitchell-reducibility holds exactly under the explicit collapse conditions of Proposition~\ref{prop:mitchell-reducibility}, and fails for genuinely regime-varying processes with irreducible memory-mediated admissibility, quotient-restricted cross-regime comparability, or non-aggregable protected cores by Proposition~\ref{prop:mitchell-obstruction} together with Proposition~\ref{prop:memory-admissibility};
    \item admissible regime evolution can preserve protected semantics and guarantee persistence under the protected stability template of Theorem~\ref{thm:protected-stability}, together with Proposition~\ref{prop:stable-persistence} and Corollary~\ref{cor:finite-chain-transport};
    \item admissible learning segments support structured morphic comparison, factorization, and composition by Propositions~\ref{prop:morphism-trajectories}--\ref{prop:morphism-invariance} and Theorem~\ref{thm:composition}.
\end{enumerate}
Consequently, the present theorem layer should be read as a first theorem-supporting synthesis of the GML learning object: it recovers the classical boundary case, isolates irreducibly multi-regime behavior beyond that boundary, and supports a disciplined theory of protected transport, comparison, and composition within one typed restriction.
\end{theorem}

\begin{proof}
Item (1) is exactly Proposition~\ref{prop:reduction}. Item (2) combines the explicit sufficient collapse conditions of Proposition~\ref{prop:mitchell-reducibility} with the structural obstruction patterns of Proposition~\ref{prop:mitchell-obstruction}; Proposition~\ref{prop:memory-admissibility} provides a concrete family of GML systems witnessing one such obstruction. Item (3) is the joint content of the protected stability template in Theorem~\ref{thm:protected-stability}, Proposition~\ref{prop:stable-persistence}, and Corollary~\ref{cor:finite-chain-transport}. Item (4) is given by the morphism, factorization, invariance, and composition results of Section~\ref{sec:representation}. Taken together, these statements establish a single theorem-supporting layer that is simultaneously recoverable at the classical boundary, irreducibly multi-regime beyond that boundary, stable under admissible transport in the precise template sense of Theorem~\ref{thm:protected-stability}, and structurally comparable and composable.
\end{proof}

\begin{corollary}[Theorem-Supporting Sufficiency of the Present Instantiation]
\label{cor:instantiation-sufficiency}
The present working theorem layer is sufficient, within a single typed restriction of the canonical GML framework, to support simultaneously four theorem-supporting roles: a classical boundary-case degeneration, a structural obstruction argument for faithful fixed-ontology reduction, the protected stability template together with explicit witness and guarantee persistence, and morphic comparison and composition results. It should therefore be read as a theorem-supporting instantiation of the framework rather than as an arbitrary illustrative example.
\end{corollary}

\begin{proof}
This is an immediate consequence of Theorem~\ref{thm:first-layer-synthesis}, which establishes all four roles in a single synthesis. The point is not that no other working theorem layer is possible, but that the present one already supports, within a single typed discipline, all four families of results listed. This is exactly what the paper requires from a first explicit instantiation: not universality, but enough internal structure to make the framework mathematically productive in a non-arbitrary way.
\end{proof}
\section{Kernel-Level Semantic Alignment}
\label{sec:kernel-alignment}

The first theorem-supporting layer of Section~\ref{sec:first-working-layer} is a numerical ordered layer. This section shows that the same foundational architecture extends to symbolic, knowledge-structured, and memory-active theorem layers~--- and that such layers can be aligned at a shared semantic kernel without collapsing their local proof calculi. This is precisely what makes GML architecturally stronger than an extension of statistical learning theory alone. A \emph{working theorem layer} is a disciplined specialization of the general GML semantic layer obtained by fixing enough additional structure to make a nontrivial family of results provable; it does not change the canonical notions of regime, protected core, admissibility, or protected equivalence. Multiple theorem layers are natural because regime plurality may alter the structure needed for proof: some problems admit ordered scalar codomains; others require partial orders, logical constraints, or memory-active admissibility. The meta-theoretical task of describing how such layers stand in relations of specialization, compatibility, and theorem transport is identified here and left for future work. The section defines semantic commensurability and proves that kernel-level guarantees transport across commensurable layers; Section~\ref{sec:examples} then grounds the symbolic case in a concrete narrative example.

\subsection{Symbolic and knowledge-structured working theorem layers}

The existence of multiple working theorem layers is not a merely technical convenience. It enlarges the foundational reach of GML. In particular, it shows that the framework is not tied in principle to one exclusively statistical semantics of evaluative improvement. A future working theorem layer may instead fix evaluator codomains, protected cores, and admissibility relations appropriate for \emph{symbolic learning}: inductive logic programming, probabilistic logic learning, relational and ontology-guided learning, and neuro-symbolic systems in which background knowledge, rules, or logical constraints directly shape learning dynamics \citep{muggleton1991inductive,deraedt2004probabilisticilp,cropperdumancic2022newintro,cropper2022ilp,hitzler2022nesy,delvecchio2025nesy,bueffbelle2024,cunnington2023nsil,pryor2023neupsl,lorello2025nesyseq,weir2024nellie}.

\paragraph{Why the symbolic extension is foundational.}
This is not an ornamental broadening of scope. ILP and related symbolic traditions already study genuine forms of learning: hypotheses are induced from examples and background knowledge, revised under counterevidence, and evaluated by notions such as entailment, consistency, satisfiability, or explanatory adequacy rather than by one scalar score alone \citep{cropperdumancic2022newintro,cropper2022ilp}. Neuro-symbolic work, in turn, shows that modern learning systems may combine differentiable components with symbolic constraints, rule extraction, logical inference, grounding, and knowledge-graph reasoning without collapsing into one proof language \citep{hitzler2022nesy,delvecchio2025nesy,bueffbelle2024,pryor2023neupsl,lorello2025nesyseq}. If GML can host such layers while preserving one semantic account of admissible learning continuity, then its foundational scope is strictly stronger than that of a framework limited to ordered numerical evaluators.

\begin{definition}[Symbolic or knowledge-structured working theorem layer]
A \emph{symbolic or knowledge-structured working theorem layer} is a working theorem layer in which at least part of the evaluator semantics, the protected core, or the admissibility discipline is realized through logical formulas, rule systems, consistency conditions, entailment relations, ontology-indexed constraints, or other explicit knowledge structures rather than through one purely ordered numerical codomain.
\end{definition}

\paragraph{Semantic reading.}
In such a layer, the protected core may contain logical constraints, consistency conditions, entailment-preserving obligations, or ontology-indexed invariants. Admissibility then cannot be reduced to numeric score change alone: it must also certify that regime transport preserves the relevant logical or knowledge-structured commitments, literally or up to protected equivalence. Evaluator compatibility changes accordingly, but the general semantic layer of GML does not. What changes is only the theorem-supporting structure fixed for proof.

This symbolic opening is not introduced as a demand that numerical and symbolic learning be fused into one local proof or optimization calculus. Its stronger and more precise role is to show that heterogeneous learning layers may remain mathematically distinct in their internal dynamics while still belonging to one and the same GML object, provided that admissible transport, protected commitments, and kernel-level commensurability are preserved.

\paragraph{Foundational consequence.}
This point is foundational rather than decorative. It means that GML can host, within one semantic architecture, working theorem layers coming from learning traditions that are often developed in isolation from one another. The present paper still develops only one theorem layer in detail. But the symbolic case shows why the architecture is stronger than a merely extended version of modern statistical learning theory: it can bridge numerical and symbolic learning-oriented regimes without conflating them, because both are treated as working theorem layers of one and the same admissibility-based semantics of learning continuity.

\paragraph{Scope and limits of the symbolic extension.}
The paper does not build the full symbolic theorem layer; it establishes the structural legitimacy of such a layer and illustrates it through the ILP example of Section~\ref{subsec:symbolic}. A fuller development, together with the meta-theory of inter-layer communication, belongs to future work. The material below makes precise the narrow sense in which distinct working theorem layers can still be aligned without collapsing them into one proof environment: the relevant unification claim is semantic, kernel-level, partial, and non-procedural. GML claims something narrower and deeper than a universal proof unification: these traditions can be treated as learning-theoretic instances of one common semantic architecture whenever they are described in terms of regimes, evaluator families, protected cores, admissible transport, and guarantee persistence.

\subsection{What numerical and symbolic theorem layers actually unify}
\paragraph{How the inter-layer results should be read.}
The inter-layer results begin with semantic alignment and only then move to restricted theorem transport. They do not claim a universal transport theorem across heterogeneous learning traditions. Their purpose is narrower: to show that once two working theorem layers share a protected semantic kernel, cross-layer comparison is well-posed at that level and some guarantees stated purely in kernel-level terms can be transported without collapsing the local proof calculi into one common formalism.

\begin{definition}[Working theorem layer]
A \emph{working theorem layer} for GML is an explicit instantiation
\[
\mathbb W=(\mathcal C_Y,\mathcal C_\Phi,\mathcal C_\Gamma,\mathcal C_{\sim},\mathcal C_T,\mathcal C_M)
\]
of the general semantic layer, where the components specify, respectively, an admissible class of evaluator codomains, protected-core realizations, admissibility realizations, protected-equivalence discipline, transport/composition discipline, and memory semantics sufficient to state a nontrivial family of results.
\end{definition}

\paragraph{Numerical versus symbolic layers.}
The ordered theorem layer developed in the present paper fixes $\mathcal C_Y$ to ordered codomains and treats comparison, factorization, and quotient arguments in that setting. A symbolic learning-oriented layer would instead allow $\mathcal C_Y$ and $\mathcal C_\Phi$ to be driven by logical validity, satisfiability, entailment, probabilistic symbolic inference, explicit proof obligations, or ontology-indexed consistency conditions \citep{cropperdumancic2022newintro,hitzler2022nesy,bueffbelle2024,cunnington2023nsil,pryor2023neupsl,lorello2025nesyseq}. The theorem layers differ, but the underlying semantic questions remain the same: what is preserved, what may change, what counts as admissible continuation, and when a guarantee survives regime transport.

\begin{definition}[Semantic kernel of a working theorem layer]
The \emph{semantic kernel} of a working theorem layer $\mathbb W$ on a subclass of GML systems is the part of the layer that specifies, independently of any local proof calculus, the interpretation of regimes, protected-core identity, admissible transport, and protected guarantee persistence on that subclass.
\end{definition}

\paragraph{Minimal requirements on semantic kernels.}
A semantic kernel is used in this paper only under three minimal requirements. First, it must make protected commitments semantically explicit rather than leaving them implicit in a local proof artifact. Second, it must preserve the meaning of admissible transport across the subclass on which the layer is defined. Third, it must anchor the interpretation of any guarantee claimed to persist under regime change. These requirements are intentionally light: they do not amount to a complete meta-theory of theorem-layer semantics or unification, but they are enough to ensure that semantic commensurability is not a purely rhetorical label.

\begin{definition}[Semantic commensurability of theorem layers]
Two working theorem layers $\mathbb W_1$ and $\mathbb W_2$ are \emph{semantically commensurable} on a subclass of GML systems if there exists a kernel-level correspondence between them that preserves the protected semantic interpretation of regimes, admissible transport, and guarantee persistence on that subclass, even though the layers may realize evaluators, comparison, or admissibility proofs differently.
\end{definition}

\phantomsection\label{prop:layer-commensurability}
\paragraph{Structural alignment of commensurable theorem layers.}
Let $\mathbb{W}_1$ and $\mathbb{W}_2$ be semantically commensurable on a shared subclass $\mathcal{S}$. By the definition of semantic commensurability, both layers assign the same protected-semantic meaning to regimes, admissible transport, and guarantee persistence on $\mathcal{S}$, even if their evaluator realizations and proof mechanisms differ. Admissibility judgments and cross-layer structural comparisons are therefore well-posed at the shared protected level; this alignment is the structural foundation on which Propositions~\ref{prop:kernel-transport} and~\ref{prop:layer-transport} rest.

\begin{proposition}[Partial transport of kernel-level guarantees]
\label{prop:kernel-transport}
Let $\mathbb W_1$ and $\mathbb W_2$ be semantically commensurable on a subclass $\mathcal S$ of GML systems. Any guarantee formulated purely at the level of their shared semantic kernel is transportable between $\mathbb W_1$ and $\mathbb W_2$ on $\mathcal S$, even if the two layers use different evaluator realizations or proof mechanisms.
\end{proposition}

\begin{proof}
A guarantee formulated purely in terms of the shared semantic kernel does not depend on any layer-specific realization, by the definition of the semantic kernel as the part of a theorem layer that specifies meaning independently of local proof calculus. Since semantic commensurability guarantees that $\mathbb W_1$ and $\mathbb W_2$ assign the same protected semantic meaning to those kernel-level objects on $\mathcal S$, such a guarantee is interpreted identically in both layers. Its transportability is therefore a direct structural consequence of the shared semantic kernel, not of any additional layer-specific deduction.
\end{proof}

\begin{definition}[Kernel-preserving correspondence]
A \emph{kernel-preserving correspondence} between two working theorem layers $\mathbb W_1$ and $\mathbb W_2$ on a subclass $\mathcal S$ of GML systems is a bidirectionally protected-faithful matching, possibly realized as a relation, that pairs theorem-layer realizations in $\mathbb W_1$ with realizations in $\mathbb W_2$ in such a way that the semantic kernel, the protected commitments, the meaning of admissible transport, and the admissible uncertainty calculus relevant to the transported guarantees are preserved on $\mathcal S$.
\end{definition}

\begin{proposition}[Kernel-level transport principle for commensurable theorem layers]
\label{prop:layer-transport}
Let $\mathbb W_1$ and $\mathbb W_2$ be semantically commensurable on a subclass $\mathcal S$ and let a kernel-preserving correspondence be fixed between them. Then every guarantee on $\mathcal S$ that is stated solely in terms of protected commitments, admissible transport, and guarantee persistence admits transport from $\mathbb W_1$ to $\mathbb W_2$, strictly up to the uncertainty calculus admitted by the target layer's certification mechanics. Layer-specific proof artifacts, comparison functionals, or local evaluator factorizations need not transport.
\end{proposition}

\begin{proof}
Kernel-preserving correspondence ensures that both theorem layers attach the same semantic meaning to the protected commitments and to the admissible continuation of learning on $\mathcal S$. A guarantee stated only in those shared terms therefore remains well interpreted after passage from one layer to the other. Crucially, such a correspondence may legitimately map an implicit protected surrogate in a numerical layer to an explicit protected constraint in a symbolic layer, provided the induced mapping preserves the protected-equivalence classes relevant to the transported guarantee. By contrast, objects whose definition depends on the local theorem calculus---for instance a particular factorization lemma, a local comparison functional, or a proof witness internal to one evaluator realization---need not survive transport, because semantic commensurability does not identify working-theorem-layer mechanisms, only working-theorem-layer meaning at the level of the kernel.
\end{proof}

\paragraph{What does and does not transport.}
The point of the theorem is deliberately partial. What transports are kernel-level guarantees: protected obligations, admissible continuation claims, and persistence statements formulated independently of one local theorem calculus. What does not automatically transport are layer-specific constructions, proofs, or optimization artifacts. This distinction is exactly what prevents theorem-layer unification from collapsing into an artificial fusion of numerical and symbolic reasoning techniques.

\paragraph{Why this strengthens GML rather than diluting it.}
The two propositions do not claim that full theorem transport is already solved. They show something more basic and more important for a foundational paper: commensurable theorem layers admit a mathematically coherent structural alignment at the protected semantic level. This is precisely what allows GML to remain sober and foundational at the same time. It does not collapse numerical and symbolic learning into one proof language, yet it gives a single semantic account of why they can still belong to one learning theory at the protected kernel level.

\subsection{Symbolic witness for kernel-level compatibility}

\begin{proposition}[Non-Vacuous Topological Realization of a Deductive Protected Core]
\label{prop:symbolic-compatibility}
There exist non-vacuous GML systems in which the protected core is defined by a hard entailment-based deductive predicate and the regime-variation cost is instantiated so that entailment-preserving transitions are assigned finite cost, while entailment-violating transitions are marked inadmissible and assigned infinite cost. Hence a discrete deductive boundary can be represented within the GML framework without forcing a continuous relaxation of the protected core.
\end{proposition}

\begin{proof}
We construct such a witness explicitly. Let the underlying logic be monotonic. Let $\State_r$ be a space of logical hypotheses and let $\M_r$ store a background theory $B_r$. Fix a safety formula $S$, and define the protected core by the predicate
\[
\Phi(V_r)=1 \quad \Longleftrightarrow \quad H \cup B_r \models S.
\]
Let $\Gamma(r,r',\tau)=1$ exactly when the transported hypothesis $\tau_S(H)$ together with the updated memory $B_{r'}$ still entails $S$, and let $\Gamma=0$ otherwise. Define the regime-variation cost by
\[
d_\Tcal(r,r')<\infty \quad \text{if } \Gamma=1, \qquad
d_\Tcal(r,r')=\infty \quad \text{if } \Gamma=0.
\]

Now consider a regime transition given by a conservative ontology extension, and let $\tau_S$ act by syntactic inclusion. If $H \cup B_r \models S$, then by monotonicity and conservativity we also have
\[
\tau_S(H)\cup B_{r'} \models S.
\]
Hence $\Gamma=1$ and the corresponding transition has finite cost. Therefore the admissible subgraph is non-empty, so the deductive realization is non-vacuous. Conversely, any transition that destroys the protected entailment is assigned $\Gamma=0$ and therefore infinite cost in the present construction. This yields a non-vacuous deductive realization of the GML protected-core and cost architecture.
\end{proof}

\paragraph{Status of the symbolic realization result.}
The point of Proposition~\ref{prop:symbolic-compatibility} is not to close a symbolic theorem layer. What it establishes is more precise: the GML semantic architecture already admits a non-vacuous knowledge-structured realization in which a hard deductive protected core and the finite/infinite admissibility boundary can be represented explicitly. The result should therefore be read as a theorem-level witness of symbolic realizability and protected compatibility, not as a claim that symbolic proof artifacts or layer-specific rule calculi already transport wholesale. We note that this symbolic witness is stated only for monotonic logical layers. For non-monotonic knowledge bases, syntactic inclusion is no longer sufficient, and admissibility would require an explicit certified belief-revision operator ensuring that the protected obligation is not retracted under the updated background theory.

\paragraph{What the symbolic opening already proves.}
The symbolic opening establishes precisely that the GML semantic architecture admits a non-vacuous knowledge-structured realization in which a hard deductive protected core and the finite/infinite admissibility boundary can be represented explicitly (Proposition~\ref{prop:symbolic-compatibility}). It does not yet build a full symbolic theorem layer for arbitrary belief-revision settings; for non-monotonic revision, additional certified belief-revision operators are required. What is already proved is the protected-level compatibility of symbolic regime evolution with the core notions of admissibility, evaluator transport, and learning continuity that govern the numerical layer. The narrative illustration for which this formal witness provides the theoretical underpinning is the symbolic ILP example of Section~\ref{subsec:symbolic}.

\section{Canonical Examples and Counterexamples}
\label{sec:examples}

The following examples ground the formal properties of GML in recognizable learning settings, showing concretely how $(\Gamma,\Phi,d_{\Tcal})$ operate on controlled instances. They are organized in three groups: numerical and boundary cases (Section~\ref{subsec:numerical}), memory-active and continual cases (Section~\ref{subsec:memory}), and symbolic cases together with a counterexample (Section~\ref{subsec:symbolic-group}). The formal symbolic witness, Proposition~\ref{prop:symbolic-compatibility}, is established in Section~\ref{sec:kernel-alignment}; the symbolic example here provides the narrative context for that result.

\subsection{Numerical and boundary examples}
\label{subsec:numerical}

\subsubsection{Boundary case: classical fixed-task learning}

Consider a single regime $r$ with:
\begin{enumerate}[leftmargin=*,itemsep=2pt]
    \item one state space $\State_r$;
    \item one memory carrier $\M_r$ that is operationally inert for semantic purposes;
    \item one evaluator $V_r$ whose protected core coincides with the evaluator itself;
    \item no nontrivial admissible transformations except the identity.
\end{enumerate}
Then the corresponding GML system collapses to a fixed-task learning process. This is the simplest witness for the reduction theorem: GML contains the classical ontology as a degenerate one-regime case.

\paragraph{Why this example matters.}
It prevents a false dichotomy. GML does not reject fixed-task learning. It explains exactly when fixed-task learning is the right boundary model.

\subsubsection{Worked toy example: two-regime linear regression}
\label{subsec:toy}

This worked example is the concrete end-to-end realization of the convex anchor-based witness established in Proposition~\ref{prop:two-regime-witness}: it makes $(\Gamma,\Phi,d_{\Tcal})$ explicit on a controlled two-regime instance where the abstract drift parameters acquire direct geometric interpretations.

Consider a two-regime linear regression problem with inputs $x\in\Real^d$ and outputs $y\in\Real$. In regime $r_0$, the learner uses a parameter vector $w\in\Real^d$ and incurs squared loss
\[
L_{r_0}(w)=\frac12\|X_{r_0}w-y_{r_0}\|^2,
\]
while in regime $r_1$ the local loss becomes
\[
L_{r_1}(w)=\frac12\|X_{r_1}w-y_{r_1}\|^2.
\]
Write $w_0^\star$ and $w_1^\star$ for the regime-local minimizers, and let the protected anchor be a fixed vector $\bar w\in\Real^d$ with trust-region radius $\rho>0$.

\paragraph{Protected core.}
The protected core is the calibration floor
\[
\Phi(V_r)=\text{``}\|w-\bar w\|\le \rho\text{''}.
\]
This core is not the whole evaluator: it isolates the part of the task identity that must survive regime change if the learning problem is to remain the same at the protected level.

\paragraph{Explicit admissible transport.}
Take the state transport to be one gradient step in regime $r_0$,
\[
\tau_S(w)=w-\eta\nabla L_{r_0}(w)
      =w-\eta X_{r_0}^\top(X_{r_0}w-y_{r_0}),
\]
for a step size $\eta>0$. Assuming that squared-loss evaluator transport is already certified across the two regimes, the transition $\tau:r_0\rightsquigarrow r_1$ is then gated by the protected retention condition. In particular,
\[
\Gamma(r_0,r_1,\tau;V_{r_0},V_{r_1})=1
\implies
\|w-\eta X_{r_0}^\top(X_{r_0}w-y_{r_0})-\bar w\|\le \rho.
\]
Thus any admissible transition must keep the transported parameter inside the protected floor. If this inequality fails, then the transport exits the protected region and the transition is inadmissible.
A more active admissible realization can also be read proximally: the transport may be interpreted as an adaptation step toward the target-regime loss followed by protected projection onto the admissible trust region. This emphasizes that, in concrete optimization-based instantiations, the protected core need not act as a passive test only; it may be enforced directly by the transport rule itself.

\paragraph{Feasibility of the admissibility condition.}
The condition $\|\tau_S(w)-\bar w\|\le\rho$ is satisfiable for any $\rho>0$. A scalar instance ($d=1$, $X_{r_0}=1$, $y_{r_0}=0$, $\bar w=0$, $\rho=1$) gives $\tau_S(w)=(1-\eta)w$, so the condition reads $|(1-\eta)w|\le 1$, which holds for all $w\in[-1,1]$ and $\eta\in(0,2)$. In the general case, a sufficient condition follows from the triangle inequality: if $\|w-\bar w\|\le\rho/2$ and $\eta\|X_{r_0}^\top(X_{r_0}w-y_{r_0})\|\le\rho/2$, then $\|\tau_S(w)-\bar w\|\le\rho$. The second half holds whenever $\eta\le\rho\,/\,(2\|X_{r_0}^\top(X_{r_0}w-y_{r_0})\|)$ for nonzero gradients. The admissibility condition is therefore non-vacuous and imposes a concrete coupling between step size, source-regime geometry, and the protected radius $\rho$.

\paragraph{Regime-variation cost.}
In the same toy setting one may instantiate the regime cost by the anchor-shift term
\[
d_{\Tcal}(r_0,r_1)=c_0\,\|w_1^\star-w_0^\star\|^2,
\]
for a layer-specific constant $c_0>0$. This matches the convex witness of Section~\ref{sec:stability}: the regime shift is controlled by the squared displacement of the regime-local anchors. Combined with the admissibility condition above, this yields a concrete transport budget in which large shifts of the local optimum force either a smaller learning step or loss of admissibility.

\paragraph{Analytical reading.}
The toy example therefore makes the formal triad $(\Gamma,\Phi,d_{\Tcal})$ explicit. The protected core is the trust-region floor around $\bar w$; the admissibility certificate is the inequality defining whether the transported iterate remains in that protected region; and the regime cost is a concrete anchor-shift budget tied to the displacement of the local optima. In particular, the admissibility inequality shows directly how the gradient step size, the source-regime geometry, and the protected floor jointly constrain admissible transport.

\paragraph{Why the example matters.}
This toy example makes three abstract points concrete. First, $\Gamma$ is not a mysterious universal oracle here: it is an explicit admissibility certificate tied to a protected trust-region requirement. Second, $d_{\Tcal}$ is not an uninterpretable formal symbol: it is a concrete transport overhead induced by regime shift. Third, a fixed-ontology scalar collapse would miss exactly the distinction between local fit improvement and protectedly admissible continuation. The example should therefore be read as a minimal worked witness of how $(\Gamma,\Phi,d_{\Tcal})$ become calculable on a controlled two-regime subclass.

\subsection{Memory-active and continual examples}
\label{subsec:memory}

\subsubsection{Continual learning with protected retention}
\label{subsec:continual}

Consider a learner evolving through a sequence of regimes $(r_t)_{t\geq 0}$, each regime corresponding to a new learning phase. The system contains:
\begin{enumerate}[leftmargin=*,itemsep=2pt]
    \item a regime-local operational state $s_t\in\State_{r_t}$;
    \item a persistent memory object $m_t\in\M_{r_t}$ storing certified retained competence summaries;
    \item a local evaluator
    \[
    V_{r_t}=(V_{\mathrm{new},t},V_{\mathrm{ret},t}),
    \]
    where $V_{\mathrm{new},t}$ measures adaptation to the current regime and $V_{\mathrm{ret},t}$ measures preservation of previously certified competence;
    \item a protected core
    \[
    \Phi(V_{r_t}) = \text{``retained competence must remain above a certified floor.''}
    \]
\end{enumerate}

A transition
\[
\tau_t:r_t\rightsquigarrow r_{t+1}
\]
is admissible only if:
\begin{enumerate}[leftmargin=*,itemsep=2pt]
    \item it is well typed on state, memory, and interface transport;
    \item the retained competence floor remains preserved, literally or up to protected equivalence;
    \item cross-regime comparison remains meaningful at the protected level;
    \item the transported memory state still certifies the legitimacy of the next adaptation step.
\end{enumerate}

Two candidate updates may achieve the same local new-regime score $V_{\mathrm{new},t+1}$ while differing in admissibility because only one preserves the protected retention floor. Thus admissibility cannot be reduced to local proxy success.

\paragraph{Why this example is foundational.}
This example shows, in concrete form, why protected cores and admissibility are not optional decorations. In continual learning, retention is often treated operationally as a regularization term, replay budget, or side constraint. In GML, retention can instead define part of the protected identity of the learning problem. This changes the ontology of the task: forgetting is no longer merely a degradation of score, but potentially a violation of admissibility.

\paragraph{Difference from Mitchell's ontology.}
A classical fixed $(E,T,P)$ description can aggregate new-task utility and retention into a single scalar criterion. But such aggregation does not by itself preserve:
\begin{enumerate}[leftmargin=*,itemsep=2pt]
    \item which part of evaluation is context-local proxy and which part is protected identity;
    \item whether a transition is legitimate because of certified memory persistence rather than accidental score balance;
    \item whether future guarantees remain meaningful under the induced transport.
\end{enumerate}
This example should therefore be read as a concrete witness for Proposition~\ref{prop:mitchell-obstruction}.

\subsection{Symbolic and boundary counterexamples}
\label{subsec:symbolic-group}

\subsubsection{Symbolic learning with protected knowledge commitments}
\label{subsec:symbolic}

Consider a symbolic learner that induces and revises rule sets from examples together with background knowledge, as in inductive logic programming and related knowledge-structured learning settings \citep{muggleton1991inductive,deraedt2004probabilisticilp,cropperdumancic2022newintro,cropper2022ilp}. Let three regimes $r_0,r_1,r_2$ correspond to successive learning phases.

\paragraph{Regime $r_0$: initial rule induction.}
The learner starts from a background theory $B_0$, a vocabulary $\Sigma_0$, and an example set $(E_0^+,E_0^-)$. The local state space $\State_{r_0}$ consists of candidate rule sets or logic programs expressible in $\Sigma_0$. The memory carrier $\M_{r_0}$ stores previously validated clauses, critical counterexamples, and protected commitments about which consequences must remain derivable or which violations must remain impossible. A natural local evaluator takes the form
\[
V_{r_0}(H)=\big(\mathrm{cov}_{0}(H),\mathrm{cons}_{0}(H;B_0),\mathrm{prot}_{0}(H)\big),
\]
where $\mathrm{cov}_{0}$ measures explanatory or predictive adequacy on examples, $\mathrm{cons}_{0}$ records consistency with $B_0$, and $\mathrm{prot}_{0}$ tracks satisfaction of protected knowledge commitments. In the symbolic reading, improvement is therefore evaluated on a partial or lexicographic order in which consistency and protected-commitment satisfaction are hard coordinates that must remain maximal for admissibility, while coverage is the coordinate that may improve strictly once those protected coordinates are preserved.

\paragraph{Regime $r_1$: vocabulary and knowledge refinement.}
Now the learner receives refined background knowledge $B_1$, an expanded or reorganized vocabulary $\Sigma_1$, and new examples. The transition is assumed to carry a signature morphism $\sigma:\Sigma_0\to\Sigma_1$ so that previously certified clauses stored in memory remain well formed under vocabulary refinement by transport $\tau_M(m_0)=\sigma(m_0)$. Some rule revisions are locally attractive because they improve coverage on the new data. However, the protected core is not identical to raw local utility. It records commitments such as:
\begin{enumerate}[leftmargin=*,itemsep=2pt]
    \item preservation of consistency with a designated protected fragment of background knowledge;
    \item non-violation of a protected rule family;
    \item maintenance of specified entailment obligations under ontology-guided refinement.
\end{enumerate}
The protected core is therefore knowledge-structured rather than purely scalar:
\[
\Phi(V_{r_1}) = \text{``protected consistency and entailment commitments''}.
\]

\paragraph{Regime $r_2$: revision under counterevidence.}
A third regime introduces counterexamples or ontology-guided disambiguation that force rule revision. The memory carrier now matters directly: it records which clauses were previously certified, which exceptions were protected, and which failures trigger admissible revision rather than semantic collapse. The learner may adopt a new local rule set $H_2$ only if the transition from $r_1$ to $r_2$ preserves the protected commitments literally or up to protected equivalence.

\paragraph{Admissible and inadmissible transitions.}
A transition
\[
\tau:r_i\rightsquigarrow r_{i+1}
\]
is admissible only if it is well typed on the symbolic hypothesis space, the memory carrier, and the relevant knowledge structures, and if it preserves the protected commitments encoded by $\Phi(V_{r_i})$. This immediately distinguishes two kinds of revision:
\begin{enumerate}[leftmargin=*,itemsep=2pt]
    \item an \emph{admissible revision}, which improves local adequacy while preserving protected consistency, non-violation, or entailment commitments;
    \item an \emph{inadmissible revision}, which may improve raw coverage or fit but destroys one of the protected knowledge commitments inherited from earlier regimes.
\end{enumerate}
Thus admissibility cannot be reduced to scalar rule quality alone. The revision step must certify that learning continuity is preserved at the knowledge-commitment level. In the symbolic layer this naturally instantiates the proof-carrying form of admissibility certificates introduced earlier: an admissible transition is not merely declared by an oracle, but is accompanied by a deductive derivation, satisfiability witness, or consistency-preservation certificate establishing that the protected commitments survive the transported revision.

\paragraph{Why Mitchell-style collapse loses the protected object.}
A fixed $(E,T,P)$ description can aggregate coverage, consistency, and rule-quality terms into one scalar objective. But such an aggregation does not preserve as first-class semantic data:
\begin{enumerate}[leftmargin=*,itemsep=2pt]
    \item the distinction between local adequacy and protected knowledge commitments;
    \item the fact that admissibility of revision depends on memory and inherited symbolic obligations rather than on score alone;
    \item the persistence of guarantee meaning across vocabulary refinement or ontology-guided regime change.
\end{enumerate}
This is exactly the point at which faithful fixed-ontology reduction loses the wrong thing: the semantic identity of the symbolic learning problem is no longer exhausted by one aggregate performance scalar.

\paragraph{Reading.}
For the theory, the symbolic example illustrates that a symbolic working theorem layer can be a genuine GML instantiation without adopting the ordered numerical evaluator structure of the main theorem layer. What is shared across numerical and symbolic cases is not one proof calculus but a semantic kernel: regimes, protected commitments, admissible transport, and guarantee persistence. Section~\ref{sec:kernel-alignment} develops the formal consequences of this commensurability (Propositions~\ref{prop:kernel-transport}--\ref{prop:layer-transport} and Proposition~\ref{prop:symbolic-compatibility}). The regime-level structure of the example is: $r_0$ establishes initial rule induction, $r_1$ introduces vocabulary refinement under protected consistency, and $r_2$ forces revision under counterevidence while preserving those commitments.

\subsubsection{Counterexample: evaluator rewriting without protected transport}

Consider an adaptive process that, after each failure, rewrites its own evaluator so that the new local score once again appears favorable. Assume:
\begin{enumerate}[leftmargin=*,itemsep=2pt]
    \item no protected core is fixed;
    \item no admissibility certificate constrains evaluator transport;
    \item no protected equivalence relation governs cross-regime comparison.
\end{enumerate}
Then the process may still exhibit apparent numerical improvement after every rewrite, but it does not define a GML learner in the present sense.

\paragraph{Why it fails.}
The process fails not because change is forbidden, but because nothing ensures that the successive evaluators still describe one coherent learning problem. Without protected transport, semantic continuity is lost. What remains is evaluatively unconstrained adaptation.

\paragraph{Why this counterexample matters.}
It blocks a vacuous reading of GML according to which any evolving adaptive system would count as learning merely because something improves locally after each modification. GML is intentionally stricter: regime change counts as learning only when it is admissible and protectedly coherent.

\subsubsection{Example from non-stationary online learning}
\label{subsec:online}

Consider an online learner facing a changing environment and evaluated by dynamic regret relative to a moving comparator. Classical theory already captures variation in the environment and in the comparator sequence. Now suppose that, in addition, a protected safety or retention requirement must remain preserved across comparator shifts. Then:
\begin{enumerate}[leftmargin=*,itemsep=2pt]
    \item dynamic regret still measures sequential performance;
    \item but it does not by itself certify that the comparator evolution preserves the protected identity of the learning problem;
    \item admissibility of comparator transport becomes a new semantic question.
\end{enumerate}

\paragraph{Why this example is useful.}
It shows that GML is not opposed to classical online learning. Rather, it identifies a layer that standard regret analysis usually leaves implicit: when does changing the reference of evaluation still count as staying inside the same learning problem?

\section{Relations to Existing Theories}
\label{sec:relations}

This section is comparative rather than adversarial. Its claim is narrow and structural: once regime variation, evaluator transport, memory persistence, and protected invariants become constitutive of the learning object, fixed-ontology learning no longer provides the right general abstraction for that multi-regime case. Table~\ref{tab:relations-summary} gives a compact overview. Six traditions are treated in full detail below: Mitchell, PAC/SLT/PAC-Bayes, continual learning, symbolic/neuro-symbolic learning, and constrained MDPs/safe RL. Five further traditions (online learning, distributional shift, domain adaptation, MDL/universal induction, and foundation models) are positioned in Appendix~\ref{sec:extended-relations}.

\begin{table}[t]
\centering
\caption{GML relative to existing learning-theoretic traditions. Each entry records what the tradition correctly formalizes, where the specifically GML boundary begins, and the exact structural relation to GML.}
\label{tab:relations-summary}
\begin{tabular}{p{0.19\linewidth} p{0.26\linewidth} p{0.26\linewidth} p{0.18\linewidth}}
\toprule
\textbf{Tradition} & \textbf{Correctly formalizes} & \textbf{Where GML begins} & \textbf{Relation to GML}\\
\midrule
Mitchell & Fixed-regime evaluative improvement & Evaluator transport; memory admissibility & Boundary degeneration\\
PAC / VC / SRM & Learnability; sample complexity & Regime-indexed risk semantics & Fixed-regime subclass\\
SLT / PAC-Bayes & Generalization; fixed risk & Guarantee transport across regimes & Regime-local subclass\\
Online / Regret & Sequential adaptation; regret & Comparator legitimacy certification & Neighborhood; overlap\\
Dataset shift & Distributional variation & Evaluator and memory semantic evolution & Statistical subclass\\
Domain adaptation & Domain-indexed distribution shift & Evaluative legitimacy transport & Partial overlap\\
Continual learning & Long-horizon adaptation; retention & Protected admissibility discipline & Natural subclass\\
Constrained MDPs & Fixed-constraint sequential learning & Regime-varying protected cores & Degenerate subcase\\
MDL / Universal & Hypothesis breadth; compression & Evaluator transport; admissibility & Largely orthogonal\\
Symbolic / ILP & Knowledge-structured learning & Cross-tradition semantic architecture & Future working layer\\
Foundation models & Multi-stage system evolution & Admissibility across deployment stages & Practical motivation\\
\bottomrule
\end{tabular}
\end{table}

\subsection{Mitchell and definitional learning}

Mitchell's definition \citep{mitchell1997}~--- together with task-transfer and meta-learning extensions \citep{blummitchell1998,baxter2000model,hospedales2021metalearning}~--- correctly captures learning under a stable semantics of performance improvement and remains the right ontology whenever one experience source, one task family, and one performance criterion suffice without losing the identity of the learning problem. Classical supervised learning and many benchmark settings fall inside this scope.

The limitation appears when evaluator transport, regime change, or memory persistence become constitutive of the learning object rather than contextual detail. Mitchell-style learning is then a strict fixed-regime boundary case of GML (Proposition~\ref{prop:reduction}): GML extends it without invalidating it. The boundary is reached when the task description must be repeatedly enriched to absorb context~--- GML insists that encoding power is weaker than semantics-preserving representation.

\subsection{PAC learning, VC theory, and SRM}

PAC learning and VC theory \citep{valiant1984,valiantpac,vapnik1998} and structural risk minimization \citep{vapnik1998,bousquet2002} correctly formalize when a learner generalizes from finite data relative to a stable error notion and hypothesis semantics. They are the natural tools whenever uncertainty is concentrated in the sample and the comparison standard remains fixed.

They do not natively formalize when a transition between evaluative regimes is legitimate, nor how guarantee meaning persists when comparison itself must be transported. They form fixed-regime or regime-local subclasses of GML; GML embeds their ontology inside a broader semantics where evaluator admissibility must itself be controlled. The boundary is reached when one attempts to preserve a guarantee while the meaning of risk itself becomes regime-indexed: sample complexity remains relevant but no longer exhausts the semantics of learning.

\subsection{Statistical learning theory and PAC-Bayes}

SLT \citep{vapnik1998} and PAC-Bayes \citep{mcallester1999,mcallester1999pacbayes,seeger2002pacbayes,catoni2007,haddouche2022onlinepacbayes} formalize generalization and uncertainty control when the learner remains inside one coherent predictive-evaluative regime. They are indispensable whenever one studies asymptotic or finite-sample behavior relative to a fixed risk functional.

They do not treat evaluator transport, protected invariance, or memory-mediated admissibility as primitive semantic constraints. SLT and PAC-Bayes are partially included in GML through regime-local evaluative theory; GML begins where the learning object is no longer exhausted by one stable statistical risk semantics. The boundary is reached when a bound certified in one regime must remain meaningful in another regime whose local proxy differs: GML adds the missing notion of protected guarantee transport.

\subsection{Continual learning, lifelong learning, and test-time adaptation}

Continual and lifelong learning \citep{caruana1997,lopezpaz2017gem,parisi2019continual} and test-time adaptation \citep{finn2017maml,liang2020tent,sun2024testtime,singh2024controllingforgetting} correctly identify that learning-relevant behavior may continue over long horizons and persist into deployment. Continual learning is one of the most natural subclasses of GML; test-time adaptation is one of its strongest motivations.

Their limitation is that retention penalties, replay mechanisms, and adaptation heuristics often leave implicit which retained properties are structurally protected, which transitions are admissible, and at what level guarantees persist. GML supplies the semantics in which retention is treated not as a heuristic penalty but as a protected evaluative core under admissible regime transport. Concretely: the memory carrier $\M_r$ corresponds to replay buffers or consolidated latent memory; $\Phi$ formalizes what EWC-style regularizers are implicitly preserving at the level of hard semantic commitments; $\Gamma$ acts as a formal checkpoint against catastrophic forgetting counting as legitimate continuation; and $d_{\Tcal}$ measures the structural severity of the task sequence. The GML boundary is reached when deployment-time change is described as continued learning and one must certify under which protected invariants that continuity claim is justified.

\subsection{Constrained MDPs and safe reinforcement learning}

Constrained MDPs and safe RL \citep{altman1999constrained,garcia2015survey} correctly capture learning governed by non-negotiable constraints rather than unconstrained reward maximization. They are the natural tools whenever the protected object is a fixed constraint set.

Their limitation is that the constraint architecture is typically fixed throughout learning. GML becomes necessary when the protected core itself evolves across regimes and one must certify that such change is admissible. Static constrained control is recovered inside GML as the special case where the protected core does not vary. The boundary is reached when the constraint architecture itself becomes regime-sensitive: the question is then no longer only whether a policy satisfies the current constraint set, but whether the transition between protected constraint regimes is a legitimate continuation of the same learning problem.

\subsection{Symbolic learning, logical constraints, and knowledge-structured learning}

This family~--- ILP, probabilistic logic learning, relational and ontology-guided learning, differentiable rule learning, and neuro-symbolic approaches \citep{muggleton1991inductive,deraedt2004probabilisticilp,cropperdumancic2022newintro,cropper2022ilp,hitzler2022nesy,delvecchio2025nesy,bueffbelle2024,cunnington2023nsil,pryor2023neupsl,lorello2025nesyseq,weir2024nellie,chengsun2025kgreasoning}~--- correctly formalizes forms of learning in which rules, logical constraints, or knowledge structures are part of what is learned, preserved, or revised. These are the natural tools whenever evaluator semantics depends on consistency, entailment, grounding adequacy, or knowledge-structured correctness conditions rather than a single scalar criterion.

Their limitation is not lack of depth but lack of a shared semantic architecture with statistical learning traditions. Each typically develops its own local semantics of revision or inductive success without a general theory of admissible regime transport and cross-regime guarantee persistence. GML addresses this precisely: it offers one semantic framework in which ordered statistical layers and symbolic learning-oriented layers are both treated as admissibility-governed forms of learning continuity, without conflating them into one theorem calculus. The boundary is reached when symbolic or neuro-symbolic learning must continue across changing background knowledge, perceptual front-ends, or protected logical commitments, and when the question is no longer only whether new rules improve local fit but whether the transition preserves the learning problem identity under admissible transport \citep{bueffbelle2024,cunnington2023nsil,pryor2023neupsl,weir2024nellie,castellano2025grounding,lorello2025nesyseq}.

\section{Scientific Utility and Practical Reach}
\label{sec:utility}

\paragraph{Explanatory and classificatory utility.}
GML makes explicit why continual adaptation, evaluator plurality, deployment-time learning, memory-governed retention, and stage-wise model evolution feel theoretically adjacent: the underlying issue in all of these is not only that something changes, but whether what changes remains part of one coherent learning process. By placing PAC, SLT, PAC-Bayes, online learning, shift theory, continual learning, and symbolic learning inside one disciplined map, GML isolates the additional semantic structure required once regime change itself enters the object of theory~--- without erasing the differences between traditions.

\paragraph{Foundational bridging and normative utility.}
GML becomes a common semantic architecture under which statistical and symbolic learning-oriented regimes can be formalized without being forced into one artificial evaluator type. The unification claim is deliberately sober: GML does not collapse these traditions into one theorem layer, but offers one foundational semantics within which both are treated as admissibility-governed forms of learning continuity~--- a frontier already demonstrated by recent ILP and neuro-symbolic work \citep{cropperdumancic2022newintro,hitzler2022nesy,bueffbelle2024,delvecchio2025nesy,cunnington2023nsil,pryor2023neupsl,castellano2025grounding}. This also yields normative consequences: typed transport, protected evaluator semantics, admissibility certification, and controlled memory evolution become explicit design obligations whenever a system is expected to learn across variable regimes.

\paragraph{Methodological utility.}
GML provides a language in which guarantees may be attached to the protected level rather than to one accidental local proxy. For modern systems whose behavior unfolds across training, post-training, retrieval-mediated use, and adaptation, the question is no longer simply whether performance goes up, but whether the route by which performance changes preserves the protected identity of the problem. Making this distinction explicit does not weaken classical theory; it locates more precisely the point at which fixed-ontology analysis ceases to be sufficient.

\paragraph{Normative and regulatory utility.}
A third dimension of utility concerns the formal treatment of ethical
and normative constraints in AI systems.
The prevailing engineering approach encodes normative obligations~---
safety, fairness, alignment with human intent~--- as scalar penalty
terms within an optimization objective, to be traded off against
task performance.
Proposition~\ref{prop:mitchell-obstruction} (condition~3) establishes
that this encoding is structurally inadequate when the normative obligation
functions as a hard admissibility condition:
replacing a non-aggregable protected core $\Phi$ with a scalar penalty
incurs irreversible structural loss, because the optimizer can then find
proxy gains outside the protected region in ways that no penalty weight
can prevent without losing the semantic identity of the problem.%
\footnote{This structural loss is the formal counterpart of the
\emph{reward hacking} and \emph{proxy gaming} phenomena documented
empirically in alignment
research~\citep{amodei2016concrete,zhi-xuan2024beyondpreferences}:
the optimizer escapes the intended constraint precisely because that
constraint has no formal structural boundary in the scalar objective.
The GML framework makes this precise: if $\Phi(V_r)$ is non-aggregable
(Proposition~\ref{prop:mitchell-obstruction}, condition~3),
no scalar reformulation is structurally faithful.}
Under the GML framework, normative obligations that must be preserved
unconditionally are encoded in $\Phi$ as non-aggregable protected
components, and their preservation is enforced by the admissibility
certificate~$\Gamma$~--- not traded off, but structurally required.

This connects to the regulatory framework now entering into force.
The \emph{EU AI Act} (Regulation~(EU)~2024/1689, in force
1~August~2024, fully applicable 2~August~2026) requires that providers
of high-risk AI systems maintain a \emph{risk management system}
as a continuous iterative process across the entire system lifecycle,
comprising the identification and mitigation of risks to health, safety,
and fundamental rights at every deployment stage~(Art.~9).
Deployers of high-risk systems must additionally assess how the system
may affect fundamental rights before each use~(Art.~27).
Both requirements presuppose that the properties to be protected are
\emph{stable across the system's transitions}~--- across updates,
retraining, post-training alignment, and capability changes.
This is precisely the structural question that GML formalizes:
which transitions preserve the protected commitments ($\Gamma=1$),
which violate them ($\Gamma=0$), and what conditions suffice to
guarantee preservation across the full operational trajectory.
GML does not itself constitute a compliance tool;
it provides the structural vocabulary~--- protected cores,
admissibility certificates, and transport conditions~---
within which lifecycle-persistent guarantees of the kind the
EU AI Act requires may be formulated rigorously and subjected
to formal verification.%
\footnote{High-risk AI systems under the EU AI Act include applications
in critical infrastructure, education, employment, law enforcement,
migration, and administration of justice (Annex~III).
The risk management obligations (Art.~9) and the fundamental rights
impact assessment (Art.~27) are the two provisions most directly
relevant to the structural lifecycle-persistence question
that GML formalizes.}

\section{Limitations, Non-objectives, and Future Work}
\label{sec:limitations}

GML is a structural theory. It does not yet provide sharp universal finite-sample bounds for every regime-varying setting, does not solve identifiability of evaluator updates in full generality, and does not give a complete algorithmic test for admissibility. The correct granularity of regimes in some application classes and the minimal algebraic structure required on regime transformations also remain open~--- deliberately, because different subclasses may require different composition laws. These are not oversights; they mark the frontier between a foundational ontology and its quantitative elaboration.

\paragraph{What `General' does and does not mean.}
The qualifier ``General'' refers to the generality of the evaluative regime under which improvement is tracked~--- not to a claim that every learning system is already captured by the present theorem-supporting layer, nor that the paper provides a complete quantitative theory of all learning systems. `General' names the semantic architecture of learning under variable regimes: protected problem identity, admissible regime evolution, evaluator transport, and persistence of guarantees under certified change. The present manuscript develops one disciplined theorem-supporting layer inside that broader architecture; it does not exhaust all admissible theorem layers or all possible quantitative instantiations of GML.

\paragraph{Computational tractability of admissibility.}
While GML fixes the semantic role of $\Gamma$ extensionally, the present paper does not assume that admissibility is trivially computable. For expressive hypothesis spaces such as deep over-parameterized neural networks, exact verification of protected-core preservation, logical entailment, or tight topological control across regimes is often computationally intractable. A central practical bottleneck for future work is therefore the construction of scalable, proof-carrying admissibility certificates $\Gamma$. Candidate approaches already studied in adjacent contexts include interval bound propagation \citep{gowal2019ibp}, abstract interpretation for neural networks \citep{singh2019abstract}, linear relaxation-based bounds \citep{zhang2018crown}, and PAC-style statistical certificates \citep{weng2018fastlin}. GML fixes the semantic gate that admissibility verification must respect; these techniques provide candidate implementations for specific regime-transition subclasses and suggest a principled agenda for making admissibility certification constructive.

\paragraph{Theorem-layer boundaries and future work.}
The manuscript develops one explicit working theorem layer, not a complete theory of all possible working theorem layers of GML. The symbolic and neuro-symbolic extension is argued for structurally and witnessed non-vacuously (Proposition~\ref{prop:symbolic-compatibility}), but its full theorematic development~--- including non-monotonic revision, a complete account of semantic kernels, and a meta-theory of inter-layer transport~--- belongs to future work. The most important next steps are fine-grained generalization guarantees under variable regimes, necessary and sufficient admissibility criteria for richer system classes, constructive algorithmic realizations of protected-core preservation, richer explicit realizations of $\Gamma$ and $d_{\Tcal}$ beyond the two-regime witness, additional working theorem layers for other regime families, and a meta-theoretical account of how such layers can be compared and aligned. These limits locate the present contribution precisely as the first rigorous treatment of admissible learning continuity under regime variation, not as its completion.

\paragraph{Toward a structural reformulation of foundational notions.}
The broader SMGI/GML programme suggests that several notions
still approached mainly through optimization objectives, surrogate losses,
and engineering heuristics may admit structural reformulation
once admissibility, protected cores, and certified typed transformation
are made explicit. The shift is not from less general to more general,
but from extensional description by performance to structural analysis
by admissible continuation.

\emph{Catastrophic forgetting}~\citep{vandeVen2024,kirkpatrick2017overcoming}
may be recast as a failure of admissible continuity under
protected memory obligations: the question is not only whether
retained performance decreases, but whether the transition violated
the structural conditions under which cross-temporal comparison
of the system's behavior remains semantically meaningful.
\emph{Neuro-symbolic learning}~\citep{garcez2023neurosymbolic,yu2025neurosymbolic}
may be recast as a problem of semantic commensurability
and certified transport across heterogeneous working theorem layers:
the open question is not only how to combine modules
but what the protected semantic core is and under which conditions
it can be legitimately transported across a representational boundary.
\emph{Alignment}~\citep{ouyang2022training,zhi-xuan2024beyondpreferences}
may be recast as the problem of certifying that the evaluative commitment
constitutive of aligned behavior is preserved under admissible regime change:
preferences are defined within a regime; the structural question is
whether the commitment they express survives legitimate capability transitions.
\emph{Safety and normative compliance}~\citep{amodei2016concrete}
may be recast as questions about what must count as protected,
how admissibility must be certified, and under which transformations
guarantees remain legitimate; certain normative obligations,
in particular, may be treated as non-aggregable protected
components of $\Phi$~--- components whose preservation cannot
be traded off against local proxy gains.
The formal basis for this claim is Proposition~\ref{prop:mitchell-obstruction}
(condition~3): replacing the admissibility certificate~$\Gamma$ with
a scalar penalty loses the structural boundary that distinguishes
admissible from inadmissible transitions, regardless of penalty weight.
This connects directly to the EU AI Act lifecycle requirements
(Art.~9 and~27, Regulation~(EU)~2024/1689): the Act mandates that
fundamental rights protections be maintained across every deployment
transition~--- a requirement structurally equivalent to requiring
$\Gamma(r_{k-1},r_k,\tau_k)=1$ on the normatively protected components
throughout the system's operational trajectory.

The present paper does not develop these reformulations as achieved results.
It provides the structural vocabulary and first theorem-supporting layer
within which they may be formulated rigorously, and identifies the
programme within which their explicit development belongs.

\section{Conclusion}
\label{sec:conclusion}

The central claim of this paper is semantic: GML is proposed as a minimal structural extension required when evaluator transport, memory persistence, regime change, and protected continuity of the learning problem become constitutive parts of the learning object. It is not proposed as a blanket replacement for existing theories, nor as a vocabulary expansion for systems already faithfully treated inside a fixed ontology. The paper identifies this learning-theoretic object, develops its first theorem-supporting layer, and shows that this layer already supports structural coherence, boundary-case degeneration to the fixed-regime setting, a structural obstruction argument for faithful fixed-ontology reduction, protected-level stability templates, and first numerical and symbolic non-vacuity witnesses.

Not all learning is multi-regime. But whenever learning genuinely
evolves through admissible regime change, the appropriate theoretical
object is no longer merely local performance improvement: it is learning
continuity under protected semantics. What remains open are richer
quantitative instantiations, fuller theorem layers, and constructive
admissibility certificates---not the coherence of the core object
established here.

There is a deeper claim behind this paper.
Classical learning theory answered the question of how a system
remains what it is through change by stipulating that the task,
the performance measure, and the evaluator are fixed;
change was then only a matter of performance within that fixed frame.
GML answers differently: a system continues to be what it is
when the evaluative structure constitutive of its learning problem
is preserved under legitimate transformation.
That shift~--- from local performance improvement within a fixed ontology
to learning continuity under protected semantics, from the theory of
the result to the theory of the continuation~--- is articulated
at the level of general intelligent systems
in the SMGI programme~\citep{osmani2026smgi}.
GML is its learning-theoretic face.
What this paper establishes is that this face has a well-defined
mathematical grammar, a first coherent theorem-supporting layer,
and the structural capacity to absorb classical learning theory
as a degenerate boundary case.

\newpage
\appendix
\section{Extended Comparative Analysis}
\label{sec:extended-relations}

The following entries complete the comparative positioning of Section~\ref{sec:relations}. Each tradition is related to GML through the same analytical lens as the main entries: what it correctly formalizes, where the GML boundary begins, and the exact structural relation.

\subsection{Online learning, shifting comparators, and dynamic regret}

\paragraph{Canonical lineage and scope.}
Online learning formalizes sequential adaptation with regret as the primary criterion \citep{cesabianchilugosi2006,zinkevich2003,auer2002,bubeckcesabianchi2012}. Dynamic regret and shifting-comparator analyses already approach regime-sensitive reasoning \citep{zhao2025efficientnonstationary}. These frameworks are the closest classical neighbors of GML whenever sequential decision quality under time-varying comparators is the central concern.

\paragraph{Relation to GML.}
There is substantial overlap. Online learning occupies a large neighborhood of GML but does not generally formalize protected-core preservation or evaluator admissibility as first-class objects.

\paragraph{GML boundary.}
The boundary is reached when the comparator sequence is not merely moving but semantically evolving: dynamic regret still measures performance but does not certify whether the comparator transport preserves the protected evaluative identity of the problem.

\subsection{Non-stationary learning, dataset shift, covariate shift, label shift, and concept drift}

\paragraph{Canonical lineage and scope.}
This family covers covariate shift, label shift, concept drift, and broader non-stationary prediction settings \citep{quionero2009dataset,zhao2025efficientnonstationary}. These literatures correctly formalize the large and important class of regime variation in which the data-generating law or temporal concept changes. They are the right abstraction whenever regime variation is principally statistical.

\paragraph{Relation to GML.}
Distribution-shift theory forms an important subclass of GML where the regime index tracks changing data laws. GML is strictly broader: it also treats evaluator and memory semantics as evolvable objects, accommodating changes that are not exhausted by distributional language alone.

\paragraph{GML boundary.}
The boundary is reached when a drift-aware method must additionally certify that the protective evaluative identity of the problem is preserved across distributional shifts~--- continuity of score under drift is not yet continuity of learning semantics under protected admissibility.

\subsection{Domain adaptation and OOD generalization}

\paragraph{Canonical lineage and scope.}
Domain adaptation and OOD generalization formalize transfer across source and target domains and robustness to unseen environments \citep{bendavid2010,ganin2016dann,arjovsky2019irm,peters2016causal,krueger2021outofdistribution}. They correctly capture domain-indexed variation in input-output distributions under a stable notion of success, and remain the right tools when domain variation is principally distributional.

\paragraph{Relation to GML.}
These frameworks are partially overlapping subclasses of GML. They fit naturally when regime variation is domain-indexed and evaluation remains protectedly stable or explicitly transportable.

\paragraph{GML boundary.}
The boundary is reached when domain transfer changes not only the data distribution but also the admissibility conditions under which predictions may still count as acceptable~--- at that point, evaluative legitimacy itself must be transported and certified.

\subsection{Algorithmic probability, MDL, and universal induction}

\paragraph{Canonical lineage and scope.}
Algorithmic probability, MDL, and universal induction formalize broad hypothesis priors, compression principles, and universal prediction under extremely rich model classes \citep{solomonoff1964,solomonoff1964a,solomonoff1964b,hutter2005}. These theories correctly study induction under maximally general representational assumptions and are central whenever maximal epistemic breadth under a fixed induction principle is the goal.

\paragraph{Relation to GML.}
The relation is largely orthogonal: universal induction broadens the hypothesis side of learning; GML broadens the semantic structure of admissible learning evolution. A universal predictor does not by itself specify when a change in evaluative regime preserves the same learning problem.

\paragraph{GML boundary.}
The missing ingredient, even for a universal predictor continuing across changing operational regimes, is not expressive power but admissibility discipline.

\subsection{Foundation models, retrieval, memory, post-training, and deployment-time adaptation}

\paragraph{Canonical lineage and scope.}
This block gathers large-scale pretraining, post-training alignment, retrieval augmentation, external memory, and deployment-time updating \citep{devlin2018,brown2020gpt3,bommasani2021,lewis2020rag,deepmind2023gemini}. This ecosystem correctly identifies that learning-relevant behavior continues beyond pretraining and may depend on retrieval, memory, alignment stages, or local adaptation at inference time. It represents the most visible practical motivation for GML.

\paragraph{Relation to GML.}
These systems are not automatically instances of GML. They become GML-relevant precisely when one asks which transitions across pretraining, post-training, retrieval-mediated behavior, alignment, or deployment adaptation are admissible and which guarantees persist across them.

\paragraph{GML boundary.}
The boundary is reached when post-training adaptation is described as continuous improvement of the same system. GML asks the missing question: under which protected invariants and admissibility constraints is that continuity claim actually justified?

\paragraph{Synthesis of the comparison.}
The comparison above supports a precise claim. Existing learning theories are not invalidated by GML. They remain the correct tools on the domains for which they were built. What changes is the ontology of the general case. Once learning continuity depends on protected evaluator structure, admissible transport, and memory-mediated regime evolution, fixed-ontology learning is no longer the full theory of that general multi-regime object. What the comparison ultimately shows is that the contribution of the paper does not lie in one isolated theorem, but in assembling these layers into a coherent learning-theoretic object: a canonical definition, a structural foundation, a first theorem-supporting instantiation, and a bounded opening toward further working theorem layers.

\bibliographystyle{plainnat}
\bibliography{jair_agi_references_v32}

\end{document}